\documentclass{article}

\usepackage{arxiv}

\usepackage{amssymb,amsmath}
\usepackage{float}
\usepackage{graphicx}
\usepackage{amsthm}
\usepackage{subcaption}
\usepackage{array}
\usepackage{hyperref}

\usepackage{algorithm}
\usepackage[noend]{algpseudocode}

\algnewcommand\algorithmicforeach{\textbf{for each}}
\algdef{S}[FOR]{ForEach}[1]{\algorithmicforeach\ #1\ \algorithmicdo}

\newcommand{\lee}{\leqslant}
\newcommand{\gee}{\geqslant}
\renewcommand{\leq}{\lee}
\renewcommand{\le}{\lee}
\renewcommand{\geq}{\gee}

\newcommand{\REM}[1]{}
\newcommand{\G}{\ensuremath{{\cal G}}}

\newtheorem{theorem}{Theorem}
\newtheorem{lemma}{Lemma}

\DeclareMathOperator\outdeg{out-deg}
\DeclareMathOperator\vlr{VLR}
\DeclareMathOperator\vs{VS}
\newcommand{\edges}[1]{#1.\textit{out-edges}}
\DeclareMathOperator\angl{angle}
\newcommand{\ang}[1]{\angl(#1)}
\DeclareMathOperator\dista{dist}
\newcommand{\dist}[1]{\dista(#1)}
\DeclareMathOperator\prede{pred}
\newcommand{\pred}[1]{\prede(#1)}
\newcommand{\preds}[1]{#1.\textit{preds}}

\title{RCS: a fast path planning algorithm for\\ Unmanned Aerial Vehicles}

\date{April 19, 2020}

\author{
	Mohammad Reza Ranjbar Divkoti \\
	Department of Computer Engineering\\ Ferdowsi University of Mashhad\\ Mashhad, Iran \\
	\texttt{mohammadreza.ranjbardivkoti@mail.um.ac.ir} \\
	\And
	Mostafa Nouri-Baygi\thanks{Corresponding author} \\
	Department of Computer Engineering\\ Ferdowsi University of Mashhad\\ Mashhad, Iran \\
	\texttt{nouribaygi@um.ac.ir} \\
}

\begin{document}
\maketitle

\begin{abstract}
Path planning is a major problem in autonomous vehicles. In recent years, with the increase in applications of Unmanned Aerial Vehicles (UAVs), one of the main challenges is path planning, particularly in adversarial environments. In this paper, we consider the problem of planning a collision-free path for a UAV in a polygonal domain from a source point to a target point. Based on the characteristics of UAVs, we assume two basic limitations on the generated paths: an upper bound on the turning angle at each turning point (maximum turning angle) and a lower bound on the distance between two consecutive turns (minimum route leg length).

We describe an algorithm that runs in $O(n^2 \log n)$ time and finds a feasible path in accordance with the above limitations, where $n$ is the number of obstacle vertices. As shown by experiments, the output of the algorithm is much close to the shortest path with this requirements. We further demonstrate how to decompose the algorithm into two phases, a preprocessing and a query phase. In this way, given a fixed start point and a set of obstacles, we can preprocess a data structure of size $O(n^2)$ in $O(n^2 \log n)$ time, such that for any query target point we can find a path with the given requirements in $O(n \log n)$ time. Finally, we modify the algorithm to find a feasible (almost shortest) path that reach the target point within a given range of directions.
\end{abstract}


\section{Introduction}

An Unmanned Aerial Vehicle (UAV), is a flying machine without a human pilot. Due to its advantages in military, civil and industrial markets, they are largely employed in situations dangerous for human life~\cite{sarris2001survey,herrick2000development,spraying2002commercial}. Since traditional remote piloting are not adequate for current complex missions, autonomous path planning are utilized extensively for UAVs~\cite{goerzen2010survey}.

Path planning is the problem of finding a feasible path from a given source to a given target for a moving object, that optimizes a desirable objective function and taking into account different constraints. The constraints are of various types, for example, non-accessible points in the space, velocity and acceleration constraints, uncertainty about the environment and so on. This problem in its generic form is NP-hard~\cite{canny1987new} and has polynomial time algorithms only in some special cases~\cite{lavalle2006planning}.

This problem has been extensively studied in robotics and various algorithms are presented for different forms of it. In general, the problem of path planning in robotics is more complex than in UAVs because of high degrees of freedom. For example, articulated robots can move their arms, or the shapes of some robots are too complex, and the robots must move in several directions to be able to pass through some regions in the environment (Piano Mover's Problem~\cite{reif1979complexity}).

In contrast to the robot path planning, the UAV path planning has several characteristics. The three major differences in path planning for UAVs are as follows. First, due to the small size of the flying object compared to the space, we usually model a UAV as a moving point in the space. This makes the problem of path planning much simpler.

Second, because of the high speed of UAVs, it is not possible for a UAV to rotate by a sharp angle in its flight path~\cite{zheng2005evolutionary,szczerba2000robust}. This is completely different from a robot's movement in which the robot can slow down or even stop to rotate. Because of this, the flight path of a UAV must be smooth and the path planner should avoid sharp turns. This limitation makes the path planning for UAVs different compared to the conventional path planning for robots, and increases the complexity of the path planner. The complexity magnified when the number of obstacles increases and the path planner may have to create a longer zigzag path with many rotations in order to make a smooth path. In addition to the limit on the turning angles, usually a UAV cannot have two consecutive turns in a short time. This enforces the path planner to make enough distance between two successive turns~\cite{zheng2005evolutionary,szczerba2000robust}.

The third difference in the UAV path planning is the ability of a UAV to change its flight height. Robots usually move on a surface, so the problem of robot path planning is normally defined on the plane. But given the fact that UAVs fly in the space, and sometimes there is a need for a change in altitude to avoid obstacles (for example, a UAV flying over a terrain, or a cruise missile over a sea near scattered islands). In these cases, the problem of designing a path should be solved in three-dimensional space which introduces new complexities~\cite{mittal2007three,roberge2013comparison,yang20083d}.

In this paper, we take the first and second characteristics of UAVs path planning into account, and, as will be explained later, design a smooth path for a moving point among polygonal obstacles. However, in order to simplify explaining the algorithm in this paper, we choose the two-dimensional space. It is not too hard to generalize the approach to three-dimensional space.

The UAVs path planning algorithms are divided into two general categories, offline and online, based on the knowledge of the planner about the environment. A path planning algorithm is called offline, if the designer has complete information about the environment and obstacles in it~\cite{mittal2007three, nikolos2003evolutionary, xue2014offline}. On the other hand, an online algorithm knows little or nothing at all about the environment in which the movement will take place~\cite{wen2015uav, wen2017online, nikolos2003evolutionary}. In the latter case, the path planning is done while the UAV moves and based on the information we get from the sensors, the path is corrected. In this paper we choose to work on offline algorithms.

In terms of the number of reported paths, the algorithms can be divided into three categories. In the first category, the path planning is performed only for one device. Most algorithms are included in this category. In the second category, the path planning is done simultaneously for a set of UAVs. In some applications, for example a salvo attack against a warship, in order to prevent air defence systems activation, target missiles must simultaneously hit the same target with different angles. In this case, the path planner algorithm finds the best paths for several UAVs all at once~\cite{yao2016cooperative, sahingoz2014generation, kothari2009multi, nedjati2016complete}. In the third category, the path planning is performed for different targets, while the environment and the source point are fixed. This type of algorithms are useful when the target point is moving and we want to find the best time/place of the target to reach it, by taking advantage of fixed environment to speed up the path planning.

In this paper, we first propose a path planning algorithm for a single UAV. Later, with some changes, we split the algorithm into a preprocessing phase and a query phase in order to be able to find paths for multiple targets in a much shorter time.

In terms of the shape of the path, the algorithms are divided into two classes. In the first one, the output of the algorithm is composed of straight line segments that are connected together~\cite{fu2012phase,szczerba2000robust,zheng2005evolutionary}. In the second class, the output consists of straight line segments and more complex curves such as arcs~\cite{roberge2013comparison}, Bezier curves~\cite{yang2015generation,sahingoz2014generation,yang20083d} and B-Spline curves~\cite{mittal2007three,nikolos2007uav}. The advantage of the former is the simplicity of the path and the speed of the path planner, while the latter is better to adapt to the rotation limitations of UAVs and not having sharp angle rotations. In contrast, the latter algorithms have high execution time and are rarely applicable in practice.

Based on the limitations described above that we consider for the path, we have the simplicity and speed of line segment path planners, while at the same time smoother paths are obtained that are traversable by UAVs.

The two limitations used in this paper, were also previously considered by other researchers. Szczerba \textit{et al.}~\cite{szczerba2000robust} were the first who described these two limitations. They used a heuristic search algorithm to find a path in the two-dimensional space with the above limitations that has length not more than a fixed given value and approaching the target from a desirable direction (approximately). The main problem with their algorithm, as will be explained in Section~\ref{Algorithm_Evaluation}, is its running time. Later Zheng \textit{et al.}~\cite{zheng2003real} and Zheng \textit{et al.}~\cite{zheng2005evolutionary} considered the problem in the three-dimensional space and added two new limitations to the path, namely minimum flying height and maximum climbing/diving angle. They used an evolutionary algorithm to solve the problem, which as Szczerba \textit{et al.}'s algorithm~\cite{szczerba2000robust} has high running time when the complexity of the environment grows.

\subsection{Our results}
We can summarize our results as follows:

\begin{itemize}
	\item We present an offline, graph based path planning algorithm, called RCS. We first make two assumptions about the required path, and then create a graph which models paths with those requirements, and finally apply a modified Dijkstra's algorithm~\cite{cormen2009introduction} on this graph to find the shortest feasible path from the source node to the target node. Because we find the shortest path with these properties, as we show in experimental results, the length of the result of the algorithm is much close to the length of the optimum path without the assumptions.
	
	\item We prove an upper bound on the time and space complexity of RCS. We show that if all inputs of the problem is given at the same time to the algorithm, we can solve the problem in $O(n^2 \log n)$ time and $O(n^2)$ space. In these upper bounds $n$ is the number of obstacle vertices.
	
	\item In another version of the problem, where several UAV's need to be launched from a fixed source, or we need to find the best time/path to reach a moving target point, we divide the RCS algorithm into two phases, such that given the set of obstacles and the source point, we preprocess a data structure of O$(n^2)$ size in $O(n^2 \log n)$ time, and for any query target point we can find the result in $O(n \log n)$ time.
	
	\item We change the RCS algorithm in such a way to find the best path that approaches the target point from a direction within a specified range of directions. The running time and the space usage of the algorithm do not change.
\end{itemize}

A preliminary version of this paper is presented in the 7th International Conference on Computer and Knowledge Engineering~\cite{ranjbar2017path}.
In this paper the running time of the algorithm is extremely improved. Furthermore, in this paper we give a complete literature review, consider an extension of the problem in Section~\ref{sec_extension_direction}, and present a thorough evaluation of the algorithms.

The remaining of the paper is organized as follows. The problem is formally described in Section~\ref{Problem_description}. In Section~\ref{Preliminaries}, we discuss the geometrical concepts needed to solve the problem. We present the algorithm to solve the problem in Section~\ref{Problem_solving_algorithm}. The extensions of the algorithm are described in Section~\ref{sec_extension} and the experimental evaluations are presented in Section~\ref{Algorithm_Evaluation}. Section~\ref{Conclusion} concludes the paper.

\section{Problem Description} \label{Problem_description}
We here formally define the problem. In the plane, a set $O=\{P_1,P_2,\dots,P_h\}$ of $h$ disjoint simple polygons, called obstacles, which totally have $n$  vertices, a source point $s$, a target point $t$ and two constant values $l$ and $\alpha$ are given. Our goal is to find a path in the exterior space of polygons of $O$, called the free space, from $s$ to $t$ with the following properties:
\begin{itemize}
	\item \textit{Minimum route leg length}: the path consists of straight line segments, each of length at least $l$.
	\item \textit{Connectedness}: the consecutive segments are connected at turning points.
	\item \textit{Maximum turning angle}: the turning angles at the turning points are at most $\alpha$.
\end{itemize}

From now on, we will refer to the above three properties as simply the \emph{path requirements}. Figure~\ref{defineproblem} depicts a set $O= \{P_1,P_2,P_3 \} $ of three obstacles, with $n=14$ vertices, source point $s$ and target point $t$. Each path from $s$ to $t$ is composed of a chain of segments. The illustrated path $\{e_1,e_2,e_3\}$ in the figure satisfies the path requirements if $l \le |e_1|,|e_2|,|e_3|$ and $\theta_1,\theta_2 \le \alpha$.

\begin{figure}[!t]
	\centering
	\includegraphics[width=2in]{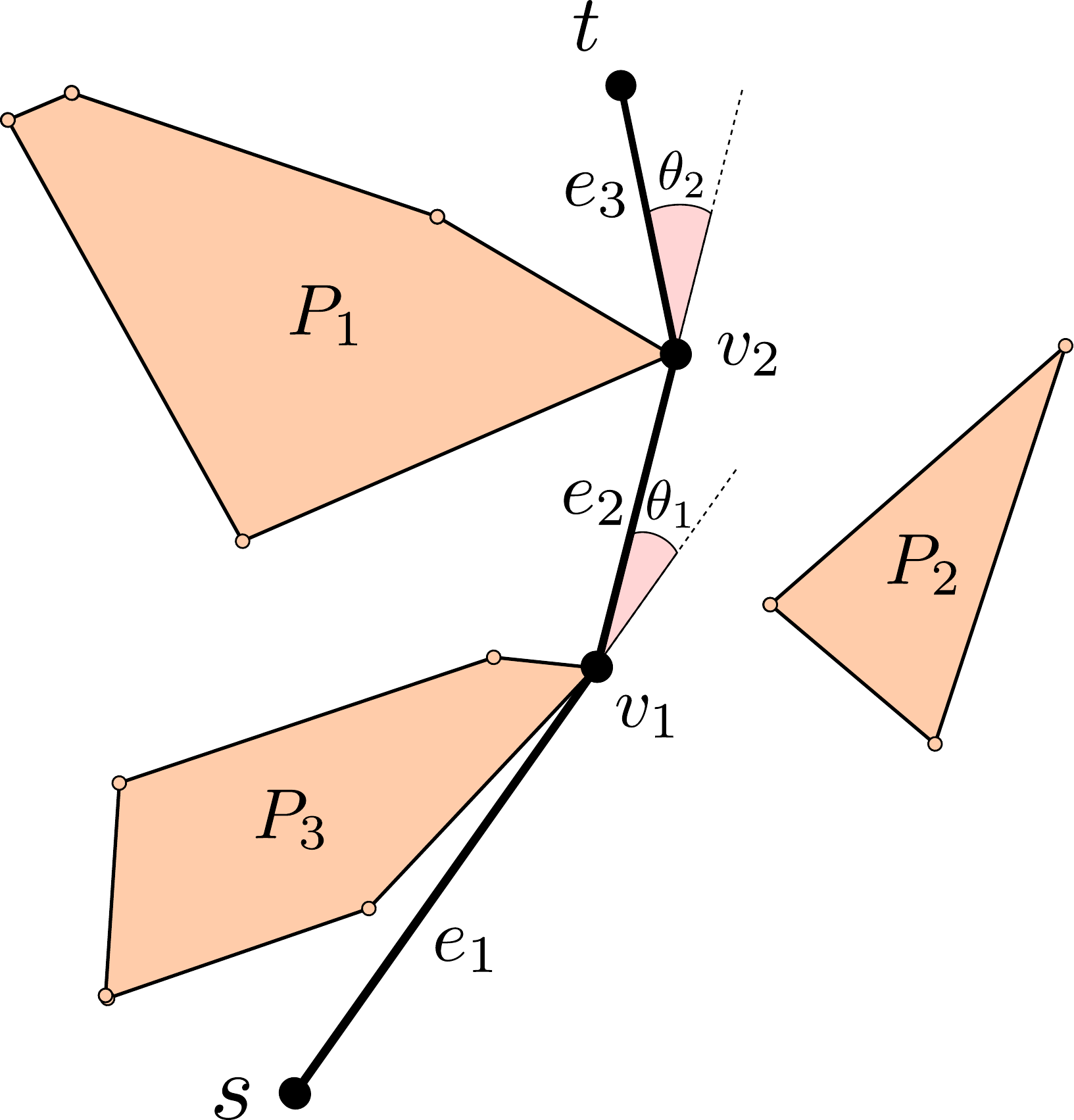}
	\caption{Path planning in the presence of polygon obstacles, with \emph{connectedness}, \emph{minimum route leg length}, and \emph{maximum turning angle} properties.}
	\label{defineproblem}
\end{figure}

\section{Preliminaries} \label{Preliminaries}
In this section, we introduce some geometric tools used for solving the problem. In Section \ref{Regular_chain_of_segments} we define a regular chain of segments and in Section \ref{Ray_shooting} the ray shooting data structure is briefly explained.

\subsection{Regular chain of segments} \label{Regular_chain_of_segments}
In our problem we consider only paths consisting of a sequence of connected line segments. We call such a sequence a \textit{chain of segments}. Each two consecutive segments have a common endpoint, called the \textit{turning point}. The angle between a segment and the extension of the previous segment in the turning point is called the \textit{turning angle}. (Figure \ref{chain_of_edge})

We call a chain of segments with all turning angles equal to $\alpha$ and equal leg lengths a \emph{regular chain of segments}. Intuitively, it is part of a regular polygon with exterior angle $\alpha$.

\begin{figure}[!t]
	\centering
	\includegraphics[width=2in]{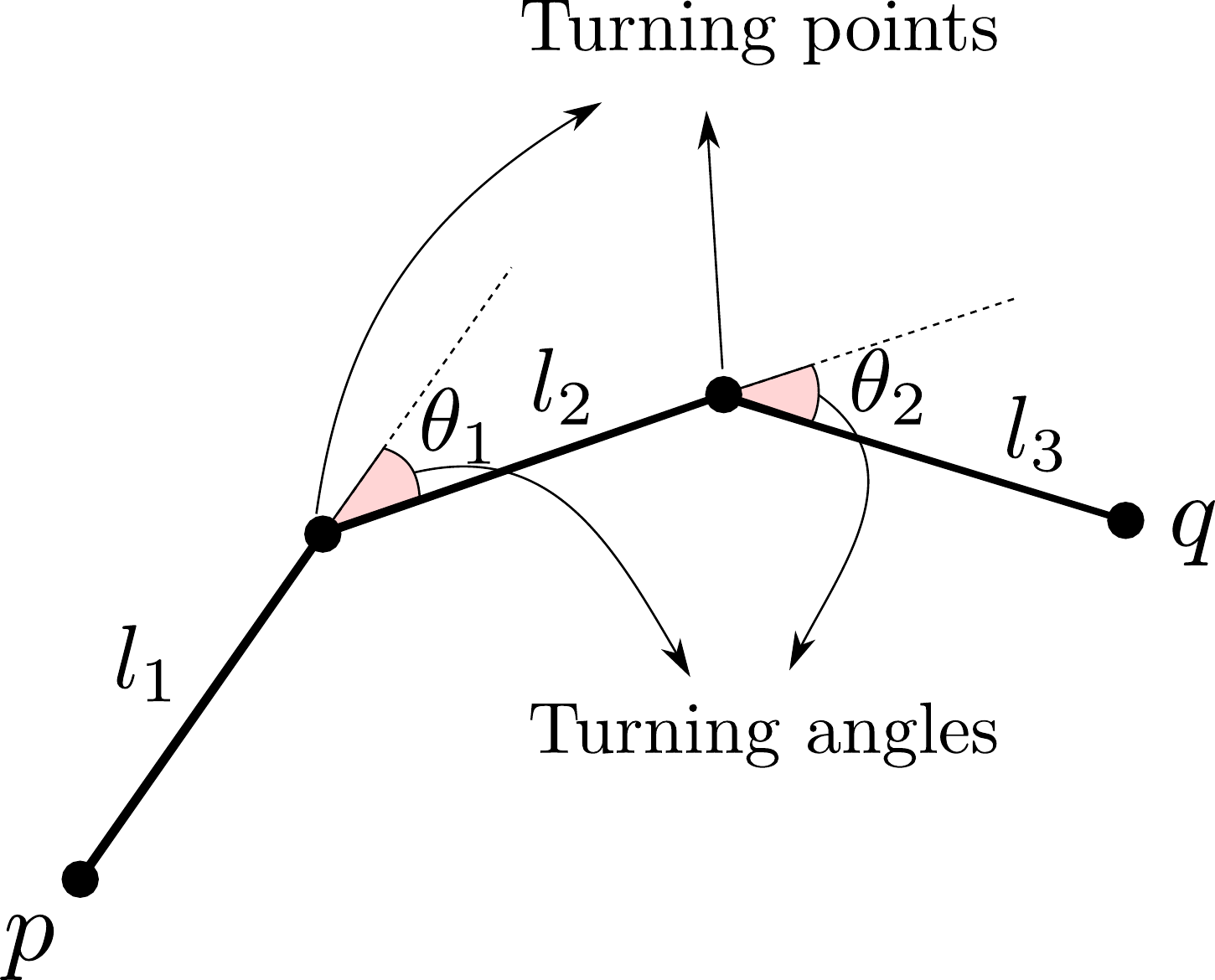}
	\caption{chain of segments from $p$ to $q$ with condition $l_1,l_2,l_3 \geq l$,  $\theta_1,\theta_2 \leq \alpha$.}
	\label{chain_of_edge}
\end{figure}

We here calculate the position of vertices of a regular chain of segments with different number of turning points, while the position of the start and the end points are known.

\begin{figure}[!t]
	\centering
	\includegraphics[width=1.5in]{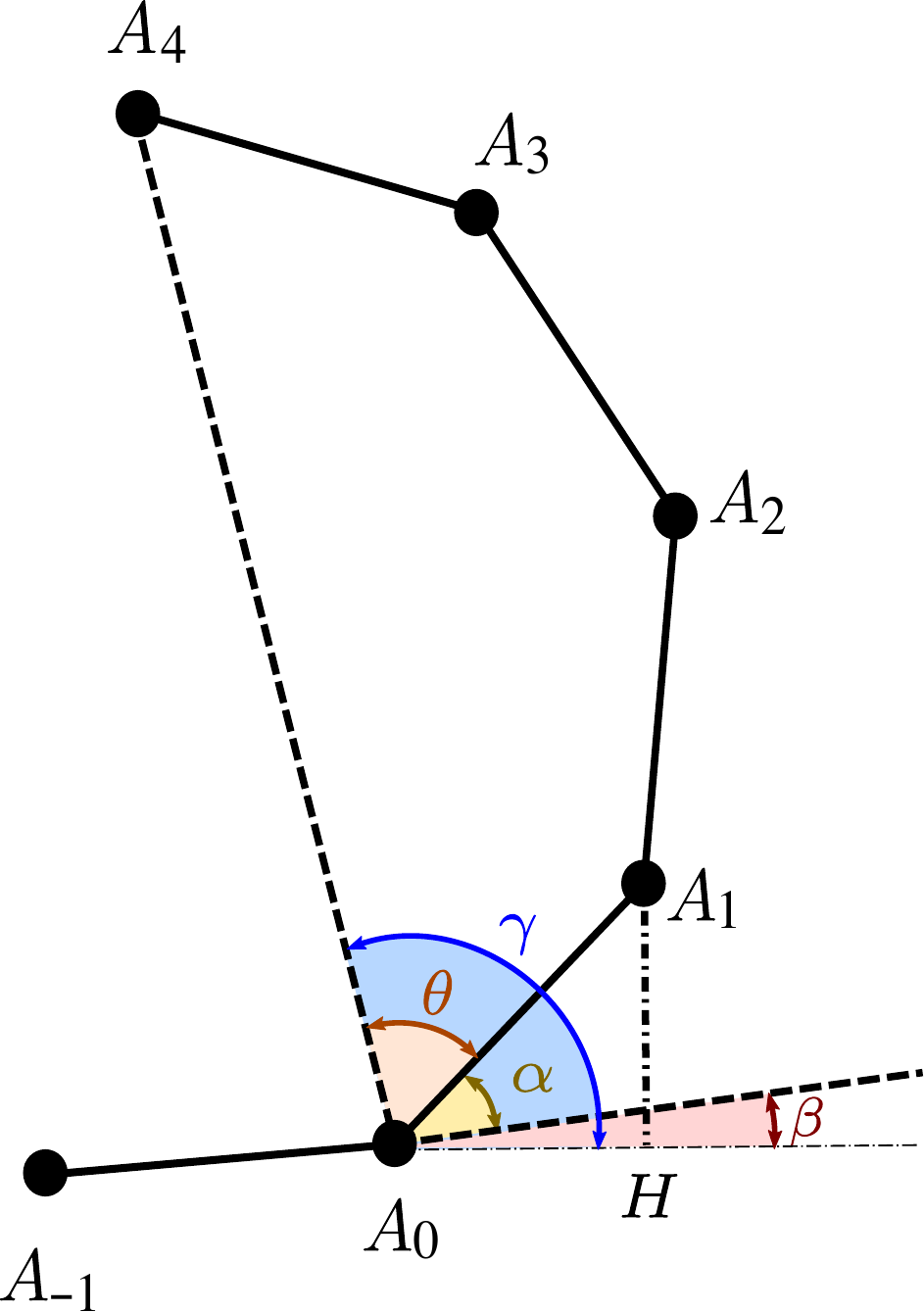}
	\caption{A regular chain of segments from point $A_0$ to $A_4$ with three turning points.}
	\label{chain}
\end{figure}

Let $C$ denote a regular chain of segments with $A_0$ as the starting point, $A_{k+1}$ as the end point, and with $k$ turning points. We assume the $x$-axis is the rightward horizontal line through the origin and the $y$-axis the upward vertical line through the origin. Let $A_i$, $1 \leq i \leq k$, be the $i^{\text{\tiny th}}$ turning point of $C$. Let $e$ denote the length of the segments of $C$. For ease of illustration, we consider $A_0$ as an imaginary turning point and denote its previous point by $A_{-1}$ (Figure \ref{chain}).  Let $\beta$ denote the angle between the extension of $\overline{A_{-1} A_0}$ and the $x$-axis. Since $C$ is a regular chain of segments, the angle between the extension of $\overline{A_{-1} A_0}$ and $\overline{A_0 A_1}$ is $\alpha$. Let the angle between $\overline{A_0 A_1}$ and $\overline{A_0 A_{k+1}}$ (the direct line segment from the starting point to the end point) be denoted by $\theta$.

Draw a line from $A_0$ parallel to the $x$-axis and project $A_1$ on this line. Let $H$ denote the projection point (Figure \ref{chain}). In the right triangle $\vartriangle A_0HA_1$ we have
$$x_{A_1} - x_{A_0} =|A_0 A_1 |  \cos{(\alpha+ \beta)},$$
$$y_{A_1} - y_{A_0} =|A_0 A_1 |  \sin{(\alpha+ \beta)}.$$
Similarly, we can draw a horizontal line through $A_1$ and project $A_2$ on this line and get the following equations
$$x_{A_2} - x_{A_1} =|A_1 A_2 |  \cos{(2\alpha+ \beta)},$$
$$y_{A_2} - y_{A_1} =|A_1 A_2 |  \sin{(2\alpha+ \beta)}.$$
As it follows, the position of each point $A_i$, can be calculated from the following recurrence relations
$$x_{A_i} - x_{A_{i-1}} =|A_{i-1} A_i |  \cos{(i\alpha+ \beta)},$$
$$y_{A_i} - y_{A_{i-1}} =|A_{i-1} A_i |  \sin{(i\alpha+ \beta)}.$$
Since $e=|A_{i-1} A_i |$, for all $1 \leq i \leq k+1$, we solve the above recurrence relations as follows
\begin{align}
x_{A_j} = e \sum_{i=1}^{j}{\cos{(i\alpha+ \beta)}} + x_{A_0},\label{eq_turn_coordinate_x} \\
y_{A_j} = e \sum_{i=1}^{j}{\sin{(i\alpha+ \beta)}} + y_{A_0},\label{eq_turn_coordinate_y}
\end{align}
where $j=1,\dots,k+1$.

In a simple polygon, the sum of the interior angles with $d$ vertices is equal to $(d-2)\pi$. If we connect $A_0$ to $A_{k+1}$, we have a simple polygon. In this polygon, the interior angle at $A_0$ and $A_{k+1}$ are equal, because $C$ is part of a regular polygon. Let $d$ be the number of vertices of the polygon. Since $k$ is the number of turning points in $C$, we have $k=d-2$ and so the sum of the interior angles of the simple polygon is $(d-2)\pi=k\pi$. The interior angle of the first and the last vertices in the chain is $\theta$ and the other angles are $\pi - \alpha$. Therefore we have $k(\pi-\alpha)+2 \theta=k\pi,$
or equivalently $\theta=\frac{k\alpha}{2}.$

Let the angle between $\overline{A_0 A_{k+1}}$ and the $x$-axis be denoted by $\gamma$. If we know the position of $A_0$ and $A_{k+1}$, we have the value of $\gamma$. Since $\gamma= \beta+ \alpha+ \theta$, we can rewrite this equation to have $\beta$ on the left hand side as $\beta= \gamma-(\alpha+\theta)$. Therefore, we have the value of $\beta$. If we further know the number of turning points of $C$, i.e. $k$, the length of segments of $C$ can be computed with each of the following equations
$$e = \frac{x_{A_{k+1}} - x_{A_0}}{\sum_{i=1}^{k+1} \cos{(i\alpha + \beta)}} = \frac{y_{A_{k+1}} - y_{A_0}}{\sum_{i=1}^{k+1} \sin{(i\alpha + \beta)}}.$$
Obtaining the value of $e$, the position of each vertex in the chain will be identified by Equations~\ref{eq_turn_coordinate_x} and \ref{eq_turn_coordinate_y}.

With the above method, given the starting and the end points and the number of turning points, we can easily determine the position of the turning points of the regular chain of segments.

\subsection{Ray shooting} \label{Ray_shooting}
One of the problems we encounter for path planning is to recognize whether a segment has any intersection with a set of segments. Formally, a set of segments $S=\{s_1,s_2,\dots,s_n\}$ are given. We are asked if segment $t$ intersects at least one of the segments of $S$.

It is obvious that this problem could be solved by checking the intersection of $t$ with each segment of $S$ in $\Theta(n)$ total time. But, if the problem must be solved a lot of times for a fixed $S$, an algorithm with less query time is more favorable.

We use a ray shooting algorithm for this problem. In the ray shooting problem, a set of segments $S$ is known. We need to preprocess $S$ such that upon receiving a query ray (its start point and direction), the first intersection position of the ray with the segments of $S$ will be detected quickly. We should also detect the case in which the ray does not intersect any segment of $S$.

By the solution of the ray shooting problem, the problem of segment intersection with a set of segments will be solved as follow: First, we preprocess $S$ for ray shooting. During the query time by receiving segment $t$, we shoot a ray from one end-point of $t$ in the direction of the other end-point. If the first intersection point of the ray with $S$ is between the two end-points of $t$, the answer to our solution is positive, that is, $t$ has an intersection with $S$. On the other hand, if the intersection point is not between the end-points of $t$ or there is not any intersection at all, the answer to our problem is negative.

There are many different algorithms for the ray shooting problem. But an old algorithm due to Chazelle~\cite{Chazelle86} is well suited to our purpose. In this algorithm, after preprocessing in $O(n^2 \log n)$ time and using $O(n^2)$ space, the ray shooting queries are answered in $O(\log{n})$ time.

\section{RCS: Path Planning Algorithm} \label{Problem_solving_algorithm}
In this section, we describe our algorithm for finding a path with the given requirements. Since we prefer to make the path smooth, we try to spread out the break points along the path evenly. Therefore, we use regular chains of segments as \emph{sub-paths} when moving from one obstacle vertex to another one, and so we call our algorithm \emph{Regular Chains of Segments Path Planner}, or simply \emph{RCS}.

Let $u$ and $v$ be two obstacle vertices. We want to find all sub-paths in the free space from $u$ to $v$ not passing through any other obstacle vertices. To do so, we examine all regular chains of segments with different number of break points starting at $u$ and ending at $v$. If all edges of a chain do not intersect any edges of the obstacles, then the chain is considered as a \emph{valid} sub-path from $u$ to $v$. In other words, we first look at the regular chain of segments without any break point, that is the straight line segment from $u$ to $v$. If the only edge of this chain does not intersect any obstacle then we consider this chain as a valid path from $u$ to $v$. Similarly, we examine the regular chain of segments with one break point, two break points, and so on. We continue this process until the distance between two adjacent break points in the chain gets smaller than $l$. If the first and the last points of a regular chain are fixed, the distance between two consecutive break points decreases when the number of break points increases (the last result of Section~\ref{Regular_chain_of_segments}). Therefore there are a limited number of valid regular chains of segments from $u$ to $v$.

Assume the moving object follows a given regular chain of segments from $u$ to $v$ when traveling from $s$ to $t$. Because of the \emph{maximum turning angle} property, when the object moves away from $v$, its \emph{leaving direction} must be within a fixed range, determined by the direction it enters $v$. This observation is depicted in Figure~\ref{fig_leaving_direction}. Following chain $c_0$, where the object moves from $u$ directly towards $v$, it can leave $v$ in the same direction and without any rotation at $v$, or rotate by an angle at most equal to $\alpha$, clockwise or counter-clockwise. Therefore, the object leaves $v$ in a direction in range $[\theta_1, \theta_2 = \theta_1 + 2\alpha]$. Similarly, for chain $c_1$ with seven break points, the valid leaving direction is in range $[\phi_1, \phi_2 = \phi_1 + 2\alpha]$; $\phi_1$ is the angle of the path if the object rotates at $v$ by an angle equal to $\alpha$ clockwise, and $\phi_2$ is the angle of the path when the object rotates by an angle equal to $\alpha$ counter-clockwise. We call this range of directions for a regular chain of segments $c$, the valid leave range of $c$, and denote it by $\vlr(c)$.

\begin{figure}[!t]
	\centering
	\includegraphics[width=2.5in]{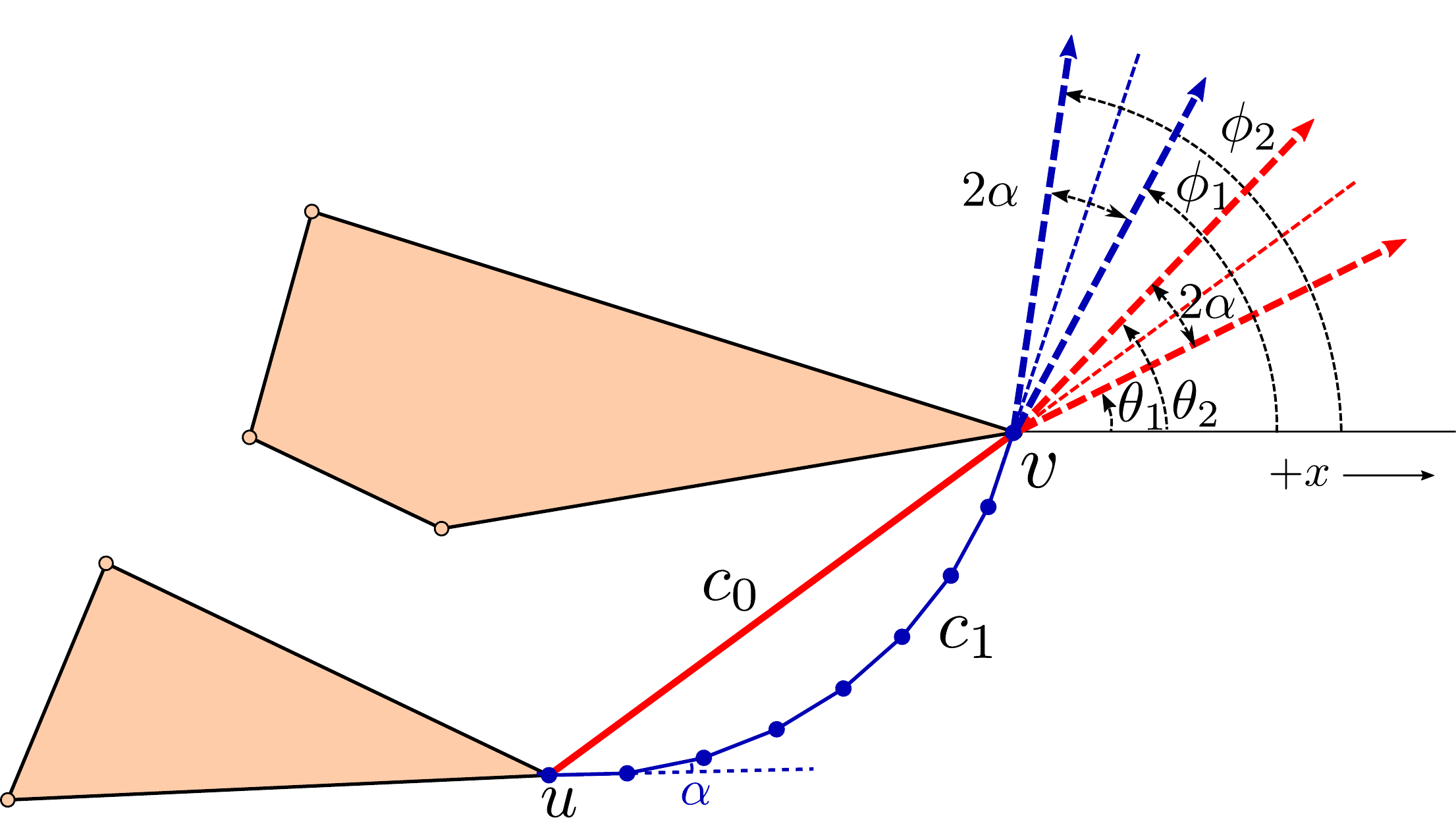}
	\caption{Following chain $c_0$, where the object moves directly towards $v$, it can leave $v$ in a direction in range $\vlr(c_0)=[\theta_1, \theta_2]$, while for chain $c_1$ the valid leaving range of directions is $\vlr(c_1)=[\phi_1, \phi_2]$.}
	\label{fig_leaving_direction}
\end{figure}

\subsection{Graph construction}\label{sec_graph_construction}
The key idea of our approach is to construct a weighted directed multigraph $\G=(V_{\G}, E_{\G})$, such that any path from the starting point $s$ to the target point $t$ in the original configuration, consisting only of regular chains of segments, corresponds to a path between analogous nodes in $\G$, of the same length. The problem will then reduce to finding the shortest path in the graph.

Graph $\G$ is constructed as follows. We add the starting point $s$ and the target point $t$ to the set of obstacle vertices. For each obstacle vertex $v$, we add a corresponding node to $\G$. For each pair of nodes $u, v$ of $\G$, and for each valid regular chain of segments $c$ between the corresponding vertices, we add two directed edges, one from $u$ to $v$, and the other from $v$ to $u$. We add two directed edges because $c$ does not have a direction and the object can follow $c$ in both directions, and each edge represents one of the directions. We say edge $e = (u, v)$ is incident from $u$ and incident to $v$, which means $u$ (resp. $v$) is the head (resp. tail) of $e$. The weight of $e$, denoted by $e.w$ is equal to the length of $c$. In addition, we store a reference to $c$, denoted by $e.c$, for further uses. It is possible to have more than one valid regular chain of segments between two obstacle vertices. So, there may be multiple edges between two nodes of $\G$, and as we indicated before, $\G$ is a multigraph.

\begin{lemma}\label{lem_graph_complexity}
For a polygonal domain of $n$ vertices, the number of nodes of $\G$, $|V_{\G}|$, is $n+2$ and the number of edges, $|E_{\G}|$, is $O(n^2)$. We can construct $\G$ in $O(n^2 \log n)$ time.
\end{lemma}
\begin{proof}
For each obstacle vertex, we add an equivalent node to $\G$. Furthermore, we add two equivalent nodes for $s$ and $t$. Therefore, the number of nodes of $\G$ is $n+2$.

For each pair of vertices $u$ and $v$, we have at most $\frac{2\pi}{\alpha}$ valid regular chains of segments. This is because the number of break points in a regular polygon with exterior angle $\alpha$ is $\frac{2\pi}{\alpha}$. For each valid regular chain of segments, two edges is added to $\G$. In summary, the maximum number of edges is ${n \choose 2} \cdot 2 \cdot 2\pi/\alpha$ which is $O(n^2)$.

For constructing the graph, we first preprocess the set of obstacle edges for ray shooting, as described in Section~\ref{Ray_shooting}. There are $n$ segments, and the preprocessing takes $O(n^2 \log n)$ time.
Then, for each pair of obstacle vertices, and each possible regular chain of segments between them, we see whether it is a valid regular chain. The validity test is done by checking the intersection of segments of the chain with obstacle edges. If any segment of the chain has an intersection with an obstacle edge, the chain is not valid. Otherwise, it is a valid chain and we add the corresponding edges to $\G$. There are ${n \choose 2} \cdot 2\pi/\alpha$ possible chains, each consisting of at most $2\pi/\alpha$ segments, and executing a ray shooting query for each one takes $O({n \choose 2} \cdot (2\pi/\alpha)^2 \log n) = O(n^2 \log n)$ total time.
\end{proof}

The above bound for the number of edges of $\G$ is clearly the worst case upper bound. This is because in most cases obstacle edges block regular chains and preclude them from being valid.

\begin{lemma}\label{lem_shortest_path}
Any path in the polygonal domain from vertex $s$ to vertex $t$ of length $l$ that consists of only regular chains of segments has a corresponding path between the analogous nodes in graph $\G$ of length $l$.
\end{lemma}
\begin{proof}
We construct $\G$ such that any valid regular chain of segments from a vertex $u$ to a vertex $v$ has a corresponding edge between the analogous nodes in the graph with the weight equal to the length of the chain. Therefore, we can establish a one to one mapping between paths in the polygonal domain which consists of only regular chains of segments, and paths in $\G$.
\end{proof}

\subsection{Finding the shortest path in the graph}
From Lemma~\ref{lem_shortest_path}, we conclude that if we need the shortest path between two vertices consisting of only regular chains of segments that satisfies some properties in the polygonal domain, we can look for the shortest path in $\G$ between the analogous nodes with the same properties.

The best algorithm for finding the shortest path between two specific nodes in a weighted graph is Dijkstra's algorithm~\cite{cormen2009introduction}, with $O(|E_\G| + |V_\G| \log |V_\G|)$ running time using Fibonacci heap data structure. But the problem here is that Dijkstra's algorithm is not directly applicable to $\G$, because not all paths in $\G$ are acceptable.

When the object moves on a path, because of the \emph{maximum turning angle} property, it could not have sharp turns at obstacle vertices. Therefore, the angle between two consecutive regular chains of segments must be less than or equal to $\alpha$. Since in Dijkstra's algorithm, edges are selected without any concern, the result of the algorithm may not satisfy the path requirements.

It is possible to transform $\G$ into a new graph and apply Dijkstra's algorithm to it to find the shortest path which satisfies the path requirements. But the large size of the new graph is another trouble. The graph has $O(|E_\G|)$ vertices and $O(\sum_{v \in V_G} \outdeg(v^2))$ edges, where $\outdeg(v)$ is the out-degree of $v$, i.e. the number of edges incident from $v$. If we apply Dijkstra's algorithm on this graph, the running time of the algorithm will be $O(|E_\G||V_\G|)$~\cite{Boroujerdi1998efficient,DECAEN1998245}, which in the worst case will lead to $O(n^3)$.

We here introduce an algorithm to find the shortest path with the path requirements, whose main idea is borrowed from Boroujerdi and Uhlmann~\cite{Boroujerdi1998efficient}. In this algorithm, no new graph is generated and we only modify Dijkstra's algorithm to find the shortest path with the given requirements. The running time of this algorithm will be $O(|E_\G| \log|V_\G|)$, which is $O(n^2 \log n)$ in the worst case.

Let $v$ be a node of $\G$ and $e$ be an edge incident from $v$. Let $\ang{e, v}$ denote the angle of the first segment of regular chain of segments $e.c$ traversed when moving away from $v$. At the beginning of the algorithm, we sort the outgoing edges of each node $v$, according to their angles, in counter-clockwise order~(see Figure~\ref{figure_out_edges}). We store the sorted list in $\edges{v}$, a data structure capable of searching and deleting items in logarithmic time, like an AVL tree.

\begin{figure*}
	\centering
	\begin{subfigure}[b]{.49\linewidth}
		\centering
		\includegraphics[width=.7\linewidth]{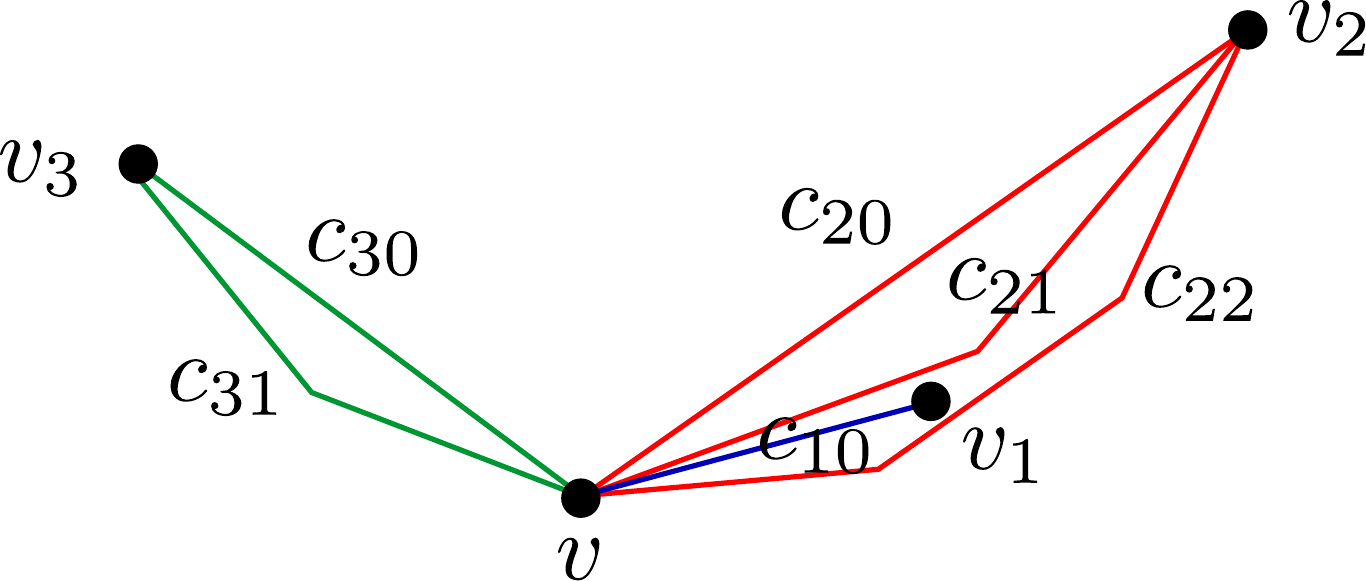}
		\caption{}
	\end{subfigure}
	\begin{subfigure}[b]{.49\linewidth}
		\centering
		\includegraphics[width=.7\linewidth]{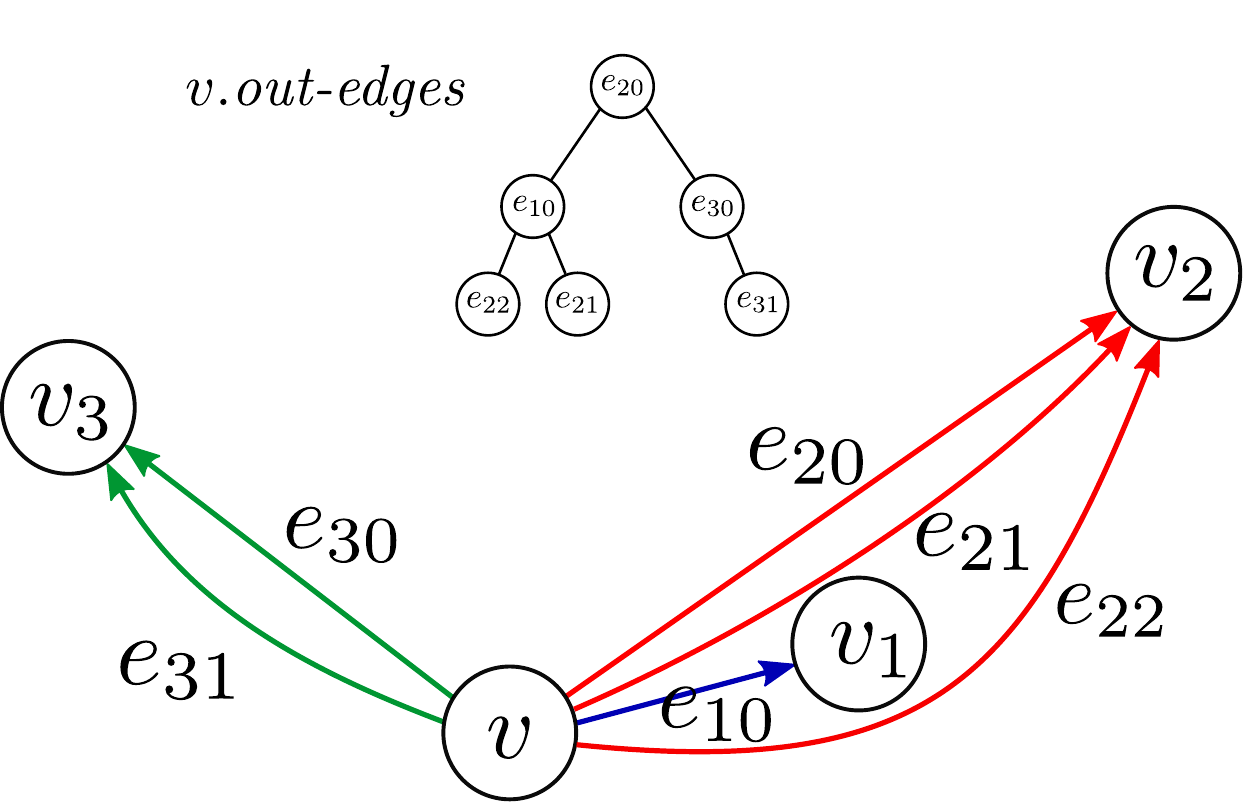}
		\caption{}
	\end{subfigure}
	\caption{Demonstration of outgoing edges of a vertex $v$. (a) Some regular chains of segments from $v$. (b) The corresponding edges incident from $v$. In tree $\edges{v}$ the outgoing edges are stored in sorted order.}
	\label{figure_out_edges}
\end{figure*}

For each edge $e$ incident to $v$, when the object arrives at $v$ through $e.c$, only a subset of the outgoing edges of $v$ are applicable for departing from $v$. The angle of the first segment of these edges must be in range $\vlr(e.c)$. We denote this valid successors of $e$ by $\vs(e)$. If the object comes to $v$ via $e.c$, it can leave $v$ following the chain of any edge in $\vs(e)$, without violating the \emph{maximum turning angle} requirement. It is easy to see that the edges of $\vs(e)$ are placed in $\edges{v}$ consecutively. The following lemma gives an upper bound on the time needed to find $\vs(e)$ for an edge $e$.

\begin{lemma}\label{lemma_vs_time}
Given a graph $\G$ such that the outgoing edges of each vertex are sorted in counter-clockwise order based on the first edge of their chains, it is possible to find $\vs(e)$ for an edge $e$ incident to $v$ in $O(\log(\outdeg(v)) + |\vs(e)|)$ time.
\end{lemma}
\begin{proof}
Let $v$ be a vertex of $\G$, and $e$ be an edge incident to $v$. Let $\vlr(e.c) = [\theta_1, \theta_2]$. We can easily find the position of $\theta_1$ (resp. $\theta_2$) in $\edges{v}$ in $O(\log(\outdeg(v))$ time. This is done by applying two binary searches with $\theta_1$ and $\theta_2$ as keys on $\edges{v}$, and reporting all edges whose angles are between the keys. The binary search takes $O(\log(\outdeg(v)))$ time, and reporting edges in $\vs(e)$ sequentially takes $O(|\vs(e)|)$ time, which proves the claim. It is necessary to note that since the outgoing edges are arranged around $v$ circularly, $\theta_2$ might be smaller than $\theta_1$. If this is the case, the report will consist of all edges whose angle is greater than $\theta_1$ \textit{or} smaller than $\theta_2$. 
\end{proof}

We now describe the modified Dijkstra's algorithm for this problem. Instead of processing nodes of the graph as Dijkstra's algorithm, we insert edges into a priority queue and process edges with minimum priority successively. For each edge $e = (u, v)$, we assign a value $\dist{e}$, whose initial value is $\infty$.  This value denotes the length of the shortest path to $v$, if the object starts from $s$ and finally passes through $e.c$. Furthermore, $\dist{e}$ plays the role of the priority in the priority queue used in the algorithm. We store another value $\pred{e}$, for each edge $e$, which denotes the last edge before $e$, in the shortest path to $e$.  The predecessor value is initially set to \textit{null}. This value helps to construct the shortest path at the end of the algorithm.

The algorithm starts by processing each edge $e$ incident from $s$, the starting point. We set $\dist{e} = e.w$, and insert $e$ into the priority queue, $Q$, with $\dist{e}$ as its priority. Now, as long as $Q$ is not empty, we remove edge $e$ with the minimum priority from $Q$. If $e$ is incident to $t$, the target point, we are done, and report $\dist{e}$ as the shortest path distance from $s$ to $t$, and $e$ and the sequence of predecessors as the shortest path. In the other case, where $e$ is not incident to $t$, but incident to some other node $v$, we process edges of $\vs(e)$, one by one. For each edge $e' \in \vs(e)$, we remove $e'$ from $\edges{v}$, set $\dist{e'} = \dist{e} + e'.w$ and $\pred{e'} = e$, and finally insert $e'$ into $Q$ with priority $\dist{e'}$. We here remove $e'$ from $\edges{v}$ because we have found the shortest path to $e'$ and $\dist{e'}$ achieves its final value. Therefore, it cannot be updated furthermore, and it is not required to be checked for other segments incident to $v$. This process continues and edges of $\G$ are explored successively until we arrive at $t$. Because $\G$ is connected, we will eventually get into $t$.

The outline of RCS Path Planning algorithm is presented in Algorithm~\ref{alg_path_planning}.

\begin{algorithm}
	\caption{RCS Path Planning}\label{alg_path_planning}
	\begin{algorithmic}[1]
		
		\Require $O=\{P_1,P_2,\dots,P_h\}$: the set of obstacles; $s$: the starting point; $t$: the target point; $l$: the minimum leg length; $\alpha$: the maximum turning angle.
		\Ensure $P$: the shortest path satisfying the path requirements, which consists only of regular chains of segment; $d$: the length of the shortest path.
		
		\Statex
		\State $V_\G \gets \text{The set of obstacle vertices} \cup \{s, t\}$ \Comment{Create the set of nodes of the graph}
		\Statex
		\State $E_\G \gets \emptyset$
		\ForEach {$u, v \in V_\G$} \Comment{Create the set of edges of the graph }
			\ForEach {regular chain $c$ between $u, v$}
			    \If {$c$ is valid}
				    \State $e \gets (u, v)$ \Comment{Create an edge equivalent to $c$}
				    \State $e.c \gets c$
				    \State $e.w \gets |c|$
				    \State $\pred{e} \gets \textit{null}$
				    \State $\dist{e} \gets \infty$
				    \State add $e$ to $E_\G$
				\EndIf
			\EndFor
		\EndFor
		\Statex
		\ForEach {$v \in V_\G$} \Comment{Compute $\edges{v}$}
			\State $\edges{v} \gets$ the list of edges incident from $v$ sorted according to the angle of the first segment of their chains in counter-clockwise order
		\EndFor
		\Statex
		\State $Q \gets$ empty priority queue \Comment{Initially the queue is empty}
		\Statex
		\ForEach {$e \in \edges{s}$} \Comment{Process the edges incident to $s$}
			\State remove $e$ from $\edges{s}$
			\State $\dist{e} \gets e.w$
			\State insert $e$ into $Q$
		\EndFor
		\Statex
		\While {$Q$ is not empty}
			\State $e \gets$ extract the minimum from $Q$
			\If {$e$ is incident to $t$}
				\State \Return $e$ and its predecessors as $P$, and $\dist{e}$ as $d$
			\Else
				\State $v \gets$ tail$(e)$
				\ForEach {$e' \in \vs(e)$}
					\State remove $e'$ from $\edges{v}$
					\State $\dist{e'} \gets \dist{e} + e'.w$
					\State $\pred{e'} \gets e$
				\EndFor
			\EndIf
		\EndWhile
	\end{algorithmic}
\end{algorithm}

\subsection{Analysis of the algorithm}
In Lines~1-11, the graph is constructed, which takes $O(n^2 \log n)$ time, as described before in Lemma~\ref{lem_graph_complexity}. Lines~12-13 sort the outgoing edges of nodes and need $O(\sum_{v \in V_\G}\outdeg(v) \log(\outdeg(v))) = O(|E_\G| \log(|V_\G|))$ time. Lines~14-18 initialize the priority queue and process the outgoing edges of the starting point, and take $O(\outdeg(s) \log(\outdeg(s)))$ time which is $O(|V_\G| \log |V_\G|)$. In Lines~19-28, each edge of the graph is inserted into $Q$ at most once, and is removed from $Q$ at most once, and removed from the list of the outgoing edges of its head at most once. These operations take $O(|E_\G| \log |E_\G|) = O(|E_\G| \log |V_\G|)$ total time. Finding $\vs(e)$ for each edge $e$, as proved in Lemma~\ref{lemma_vs_time}, takes $O(\log |V_\G| + |\vs(e)|)$ time. Since each edge is considered only in one $\vs(e)$, and is removed from the list of outgoing edges of its head, the total time spends for lines~25-28 is $O(|E_\G| \log |V_\G|)$. Therefore, the whole running time of the algorithm will be $O(n^2 \log n + |E_\G| \log |V_\G|)$, which is $O(n^2 \log n)$. In the algorithm, the space is used for storing the ray shooting data structure, $\G$ and $Q$, which is $O(n^2 + |E_\G|) = O(n^2)$.

The following theorem summarizes our results about finding the shortest path consisting of regular chains of segments.

\begin{theorem}\label{thrm_dijkstra_complexity}
The RCS algorithm finds the shortest path from $s$ to $t$, which consists of regular chains of segments, and satisfies the path requirements, in $O(n^2 \log n)$ time using $O(n^2)$ space.
\end{theorem}

\section{Extensions}\label{sec_extension}
In this section we consider two extensions of the problem and try to use our technique to solve these extensions. In the first problem there are several target points that are to be processed with a single starting point, and we have to find the best path to reach each target point. In the second extension, we solve the problem assuming that we have to arrive at the target point from a given range of directions.

\subsection{Problem in the query mode}
In the original problem, it was assumed that the target point $t$ is given at the same time as the other inputs. In this section we divide the algorithm into two parts, a preprocessing phase, and a query phase. The set of obstacles and the starting point are given in the preprocessing phase, and in the query phase the target point is determined. The problem in this mode is suitable when we want to find the best time/path to reach a moving target point in a fixed environment.

We solve the problem in the query mode as follows. Given the set of obstacles and the starting point in the preprocessing phase, we construct graph $\G$ and run Dijkstra's algorithm on $\G$ to compute $\dist{e}$ for all edges of $\G$. An edge $e = (u, v)$ is a candidate predecessor for each edge $e'$ incident from $v$ if $\ang{e', v}$ lies in $\vlr(e.c)$. It is the actual predecessor of $e'$, if among all such candidates, $\dist{e}$ has the smallest value. With this in mind, for each node $v$, we keep a sorted list of ranges around $v$ such that the predecessor of edges in each range is fixed. We store this list, denoted by $\preds{v}$, in an AVL tree, so that the insertion and searching takes logarithmic time. This tree can easily be constructed, as Dijkstra's algorithm is run on $\G$. At the beginning, $\preds{v}$ is empty for each node $v$. When the first edge $e$ incident to $v$ is processed, $\vlr(e.c)$ is added to $\preds{v}$ and $e$ is linked to it. During the execution of the algorithm, when a new edge $e'$ incident to $v$ is being processed, it may add a new range to $\preds{v}$. Since the edges are processed in increasing order of their distance, $e'$ is predecessor only in a range where no previous predecessor is assigned. Furthermore, for each edge $e$, $\vlr(e.c)$ covers a range of size $\alpha$. Therefore, for each edge $e$, if $e$ is predecessor for some range, it will be a connected range, which can easily be found and inserted to $\preds{v}$ in $O(\log n)$ time. Since the number of incoming edges to a vertex $v$ is $O(n)$, the number of ranges in $\preds{v}$ is also $O(n)$.

In the query phase, given the target point $t$, we consider each node $v$ in $\G$ and find all valid regular chains of segments from $v$ to $t$. We here need to execute a ray shooting for each segment of a chain to see if it is valid. For each valid chain, we add an edge $e$ from $v$ to $t$ in $\G$. Furthermore, we find the predecessor of $e$ using $\preds{v}$ in $O(\log n)$ time and accordingly assign $\pred{e}$. Finally we add the distance of $s$ to $\pred{e}$ to the length of $|e.c|$ and store it into $\dist{e}$.
After processing all edges to $t$, the final solution is an edge $e$ incident to $t$ with the smallest value of $\dist{e}$, along with the sequence of predecessors.

\begin{theorem}
For the problem of path planning with the path requirements, we can preprocess a data structure in $O(n^2 \log n)$ time using $O(n^2)$ space, to answer the query for any target point $t$ in $O(n \log n)$ time. In these upper bounds $n$ is the total number of obstacle vertices.
\end{theorem}
\begin{proof}
We construct the graph in $O(n^2 \log n)$ time using $O(n^2)$ space and run Dijkstra's algorithm in $O(|E_\G| \log |V_\G|)$ time as proved in Lemma~\ref{lem_graph_complexity} and Theorem~\ref{thrm_dijkstra_complexity}, respectively. In the query time, we need to find and process each edge to $t$ in $O(\log n)$ time to find a possible path to $t$. Since there are at most $O(n)$ edges to $t$, processing all edges and choosing the shortest path to $t$ takes $O(n \log n)$ time.
\end{proof}

\subsection{Approaching from a given direction}\label{sec_extension_direction}
In some path planning applications, it is desirable to reach the target from a given direction. For example the most damage will take place when a missile hits the target orthogonally, or a UAV landing on an aircraft carrier needs to approach it from a fixed direction. In this section we describe how we can use the method to achieve this goal.

In this version of the problem, in addition to the previous inputs, a range of directions like $\Theta = [\theta_1, \theta_2]$ is given, and we want to approach the target from a direction in range $\Theta$.

In Section~\ref{Problem_solving_algorithm} we described how to construct a graph $\G$ that models paths only consisting of regular chains of segments. Each edge $e = (u, v)$ in this graph has the following meaning: There exists a regular chain of segments $e.c$ from obstacle vertex $u$ to obstacle vertex $v$ which leaves $u$ and arrives at $v$ in a fixed direction determined by $e.c$.

Therefore, the solution is as follows. We continue the while loop of the algorithm until the edge that is processed comes into $t$ with an angle in range $\Theta$, or $Q$ becomes empty. It is obvious that in the former case the resulting path reaches $t$ in a direction in the desirable range, and in the latter case, there is no path with the given requirements.

By the above argument we conclude the following theorem:
\begin{theorem}
In the problem of path planning with a fixed approaching direction, a path from the starting point $s$ to the target point $t$ with the required properties can be found in $O(n^2 \log n)$ time using $O(n^2)$ space, where $n$ is the total number of obstacle vertices.

Furthermore, we can preprocess a data structure in $O(n^2 \log n)$ time using $O(n^2)$ space, to answer the query version of the problem for any target point $t$ in $O(n \log n)$ time.
\end{theorem}

\section{Algorithm Evaluation} \label{Algorithm_Evaluation}
Even though the running time and the space usage of the proposed algorithms are bearable, the dependency in the upper bounds is on the number of obstacle vertices and not on the scene size. In most applications the area in which the algorithm works is extremely large. But the complexity of the scene in terms of the number of obstacle vertices is not considerable. This can be made more rigorous if we apply a simplification step on the input to decrease the number of obstacle vertices.

In this section we perform several experiments to evaluate the RCS algorithm and show its effectiveness in practice. We compare the results to the results of the algorithm designed by Szczerba \textit{et al.}~\cite{szczerba2000robust}. We chose their algorithm because they solved the most similar problem to the problem considered in this paper. Their algorithm is based on the A* search algorithm, and we called it A* hereinafter. The evaluations are performed based on seven test cases. The source codes of the RCS and A* algorithms are available at the GitHub repository hosting service~\cite{ranjbarGithubRCS,ranjbarGithubAStar}.

In the first experiment, we run RCS and A* on a fixed complex polygonal domain from a given source point to several target points. The configuration is depicted in Figure~\ref{fig_case1_overall}. In this figure, $s$ is the source point and $t_1,\dots,t_5$ are five different target points. We use the maximum turning angle $\alpha = \pi/6$ and the minimum leg length $l = 50$ as the parameters for this test case. The target points are selected such that paths to each one have different number of turning points.

\begin{figure}[!t]
	\centering
	\includegraphics[width=7cm]{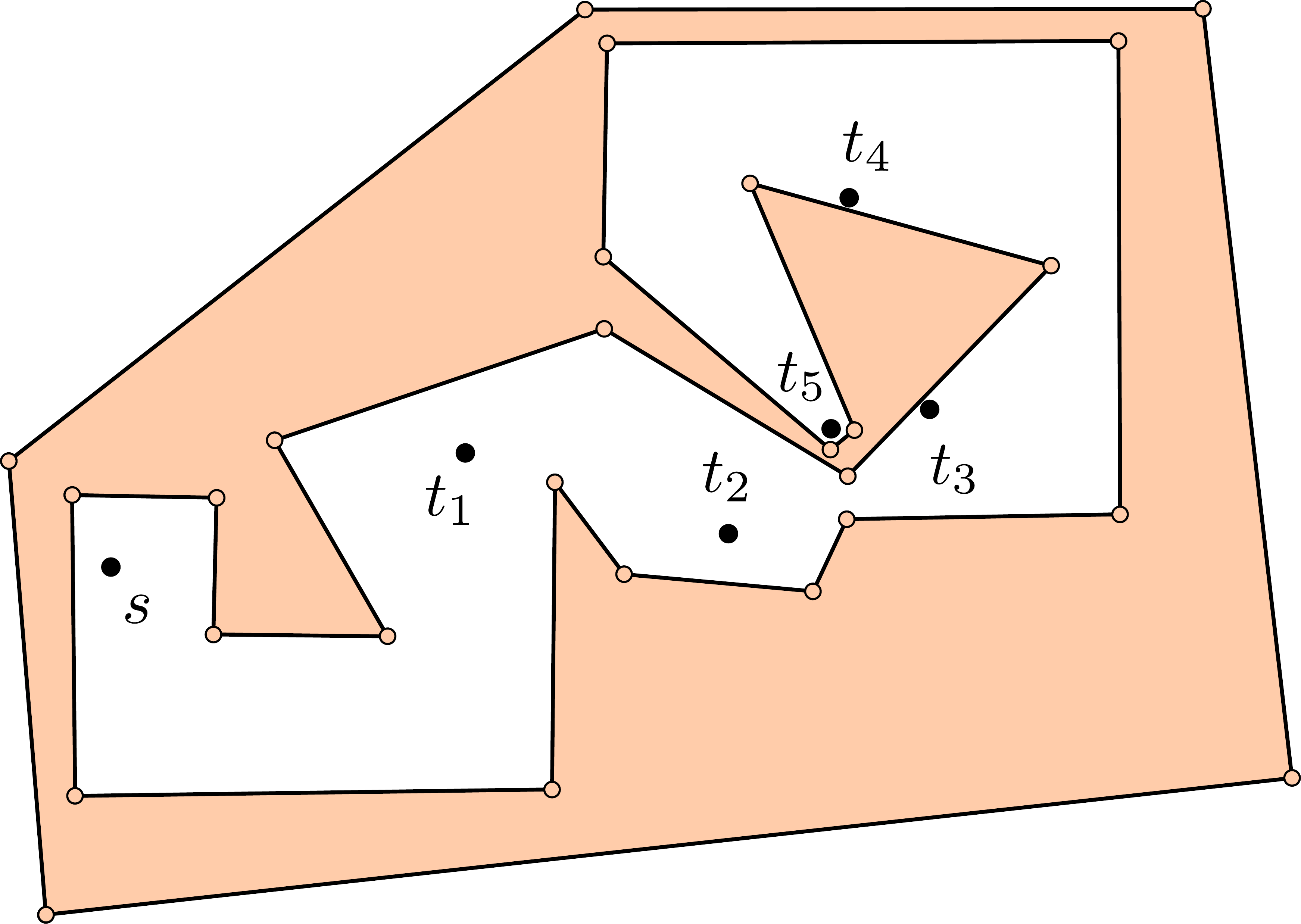}
	\caption{The scene used for the first experiment with $\alpha=\pi/6$ and $l=50$. The source point $s$ and five different target points $t_1,\dots,t_5$ are denoted in the figure.}
	\label{fig_case1_overall}
\end{figure}

\newcolumntype{C}[1]{>{\centering\let\newline\\\arraybackslash\hspace{0pt}}m{#1}}

\begin{table*}[!t]
	\caption{Comparison of the results of the RCS and A* algorithms on the configuration of Figure~\ref{fig_case1_overall}}
	\label{table_result1}
	\centering \scriptsize
	\begin{tabular}{|C{0.9cm}|C{1.7cm}|C{1.5cm}|C{1.5cm}|C{1.5cm}|C{1cm}|C{1cm}|C{1.5cm}|}
		\hline
		{Target} &
		{RCS\newline preprocessing time (ms)} &
		{RCS query time (ms)} &
		{RCS total time (ms)}  &
		{A* running time (ms)} &
		{RCS path length} &
		{A* path length} &
		{Relative difference} \\
		\hline
		$t_1$ & 6.28 & 0.26 & 6.54 & 375    & 487.6  & 470.4  & 3.5\%\\
		$t_2$ & 6.76 & 0.20 & 6.96 & 14183  & 715.5  & 670.2  & 6.3\%\\
		$t_3$ & 6.83 & 0.23 & 7.06 & 71164  & 998.5  & 952.2  & 4.6\%\\
		$t_4$ & 6.86 & 0.25 & 7.11 & 165713 & 1539.6 & 1467.0 & 4.7\%\\
		$t_5$ & 6.88 & 0.26 & 7.14 & 230670 & 1954.4 & 1860.8 & 4.8\%\\
		\hline
	\end{tabular}
\end{table*}

\begin{figure*}
    \centering
    \begin{subfigure}[b]{.49\linewidth}
    	\centering
    	\includegraphics[width=.9\linewidth]{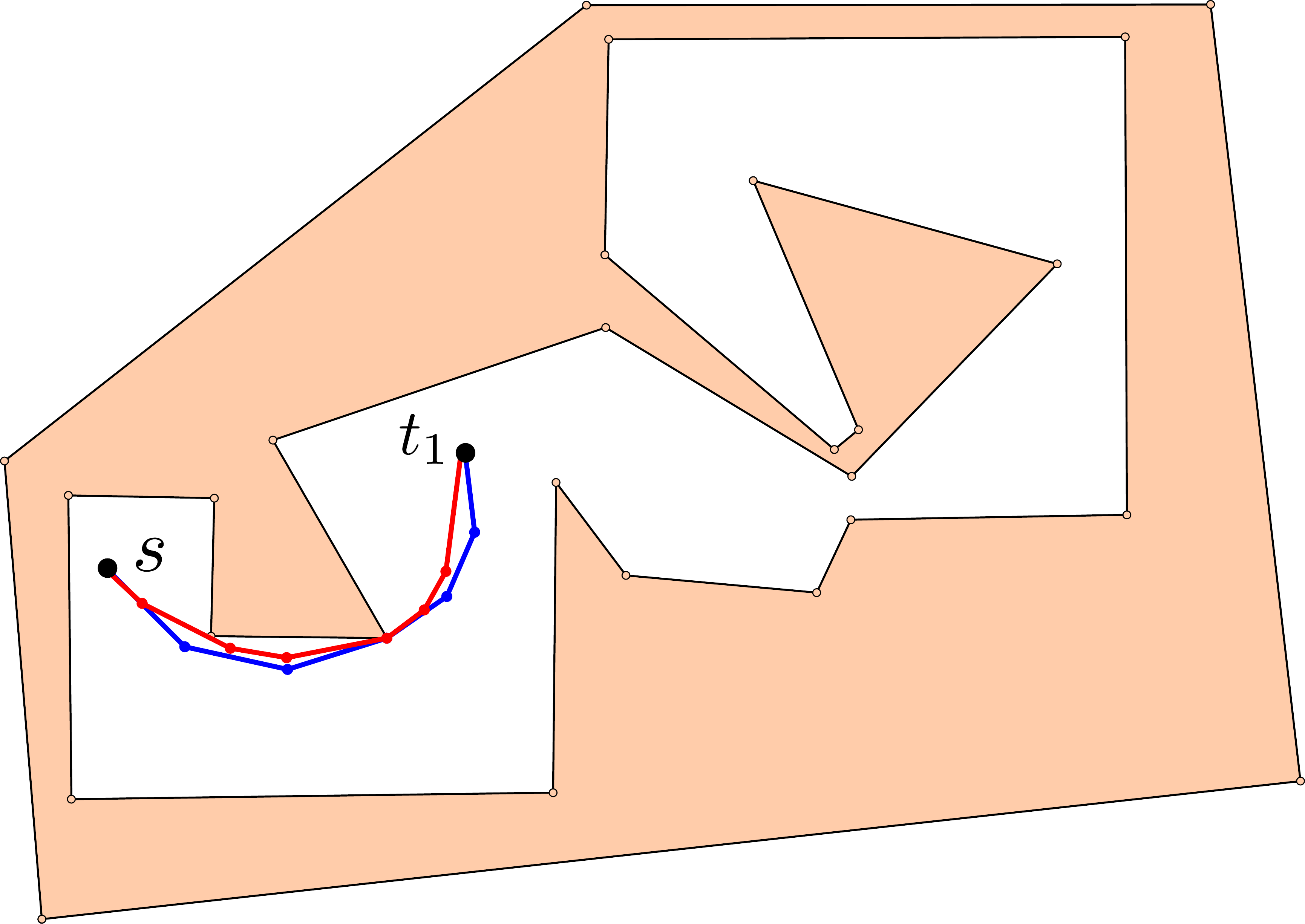}
    	\caption{}
    \end{subfigure}
	\begin{subfigure}[b]{.49\linewidth}
    	\centering
		\includegraphics[width=.9\linewidth]{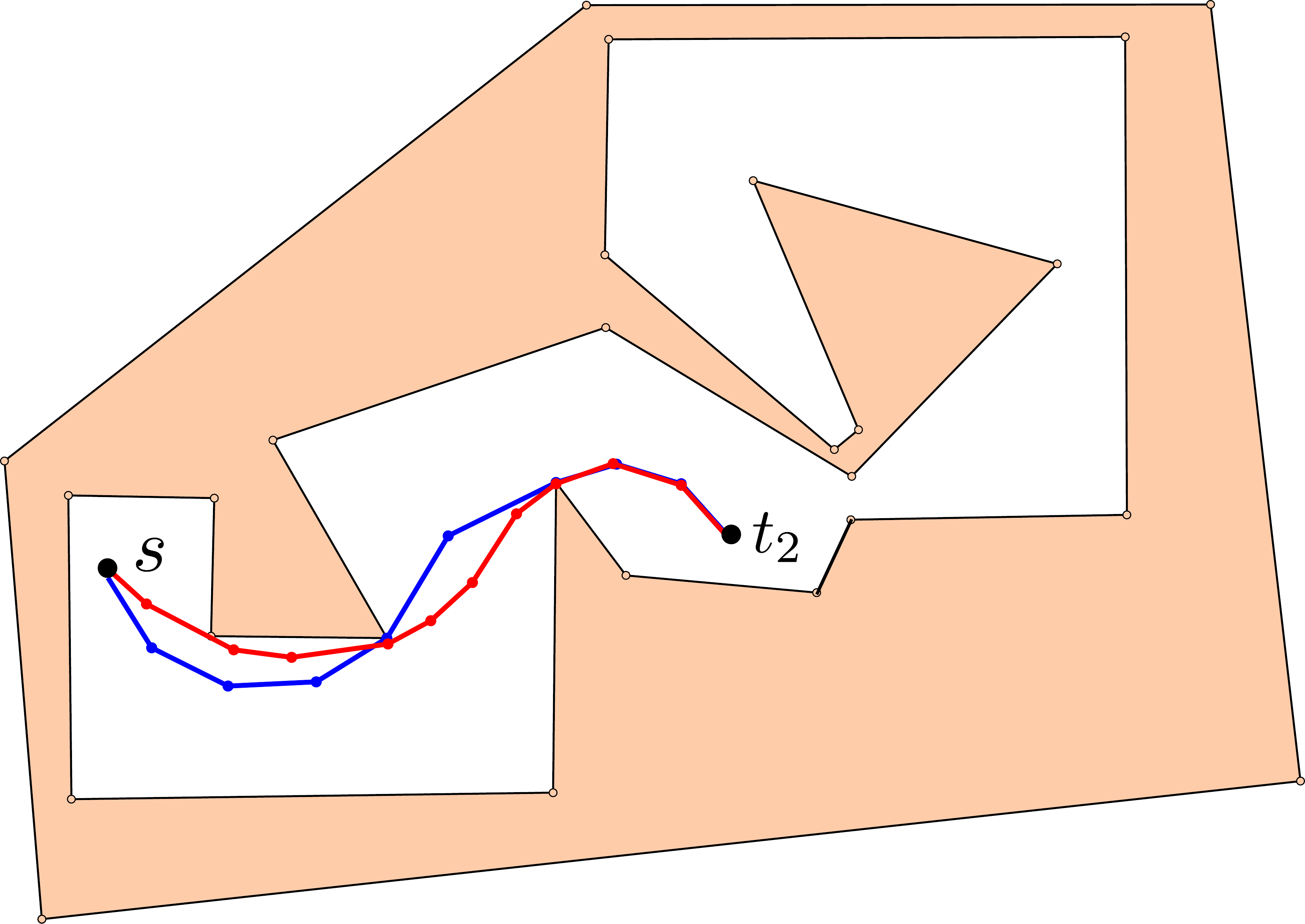}
    	\caption{}
    \end{subfigure}

\bigskip

    \begin{subfigure}[b]{.49\linewidth}
    	\centering
		\includegraphics[width=.9\linewidth]{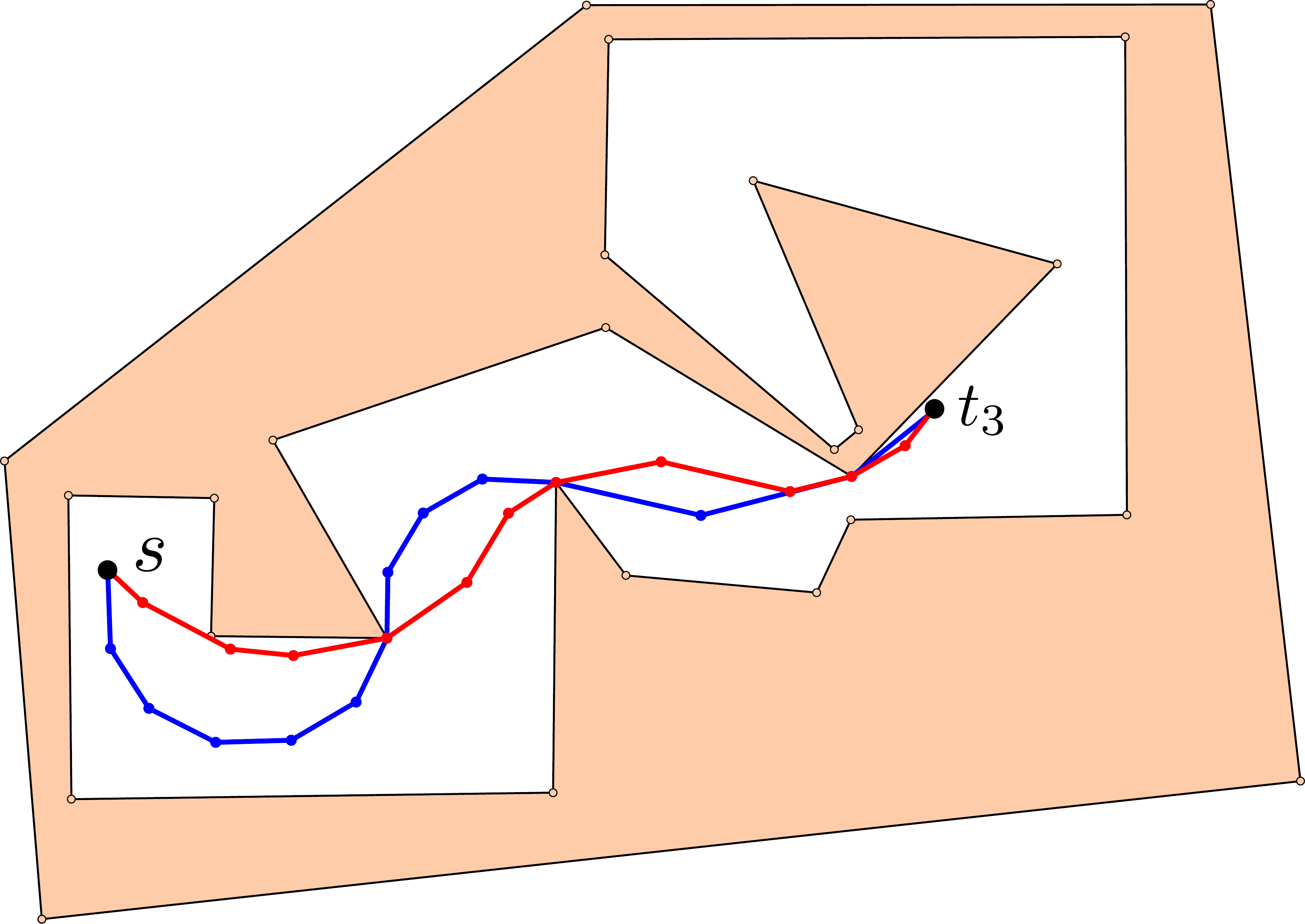}
    	\caption{}
	\end{subfigure}
	\begin{subfigure}[b]{.49\linewidth}
    	\centering
		\includegraphics[width=.9\linewidth]{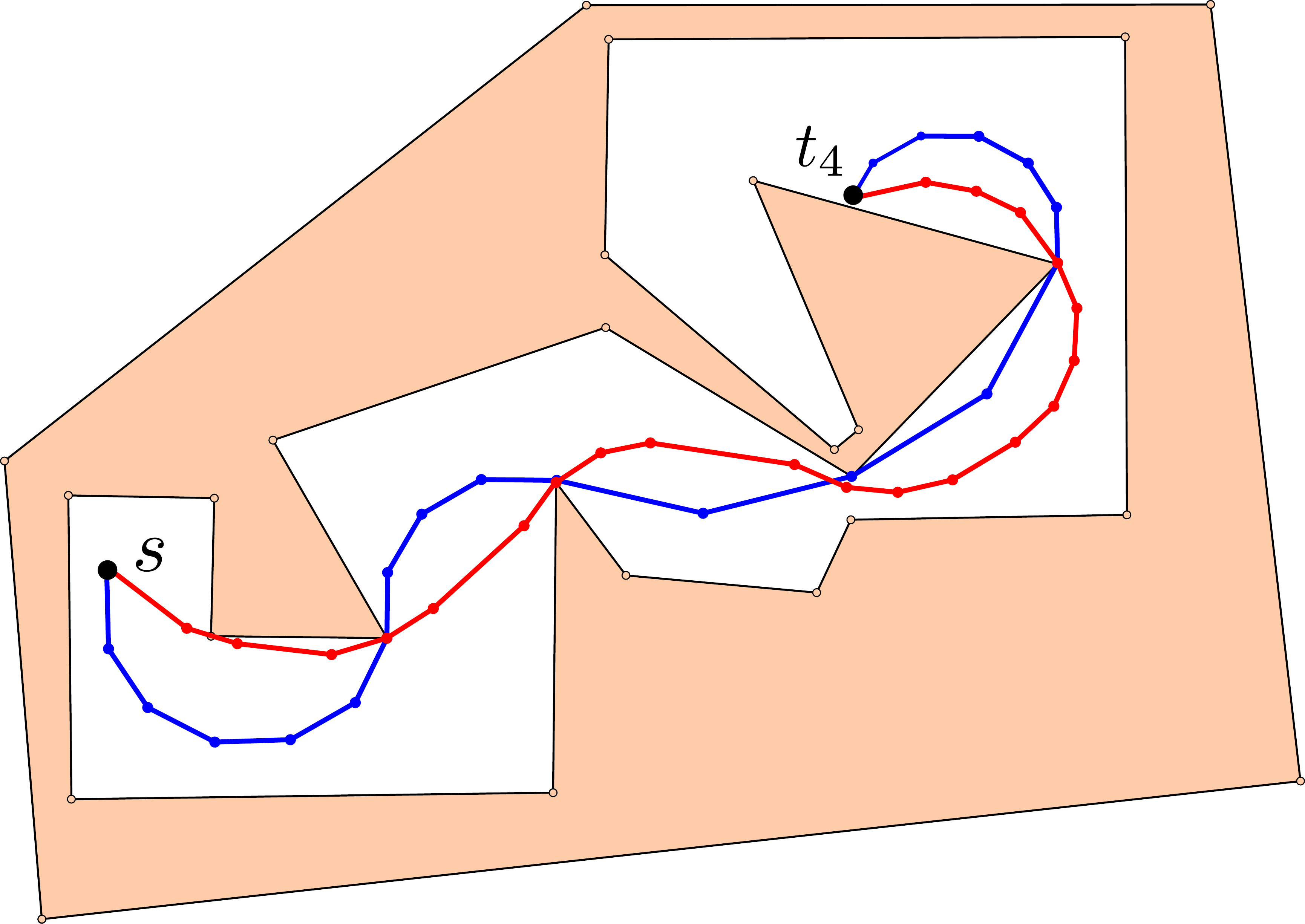}
    	\caption{}
	\end{subfigure}

\bigskip

    \begin{subfigure}[b]{.49\linewidth}
    	\centering
    	\includegraphics[width=.9\linewidth]{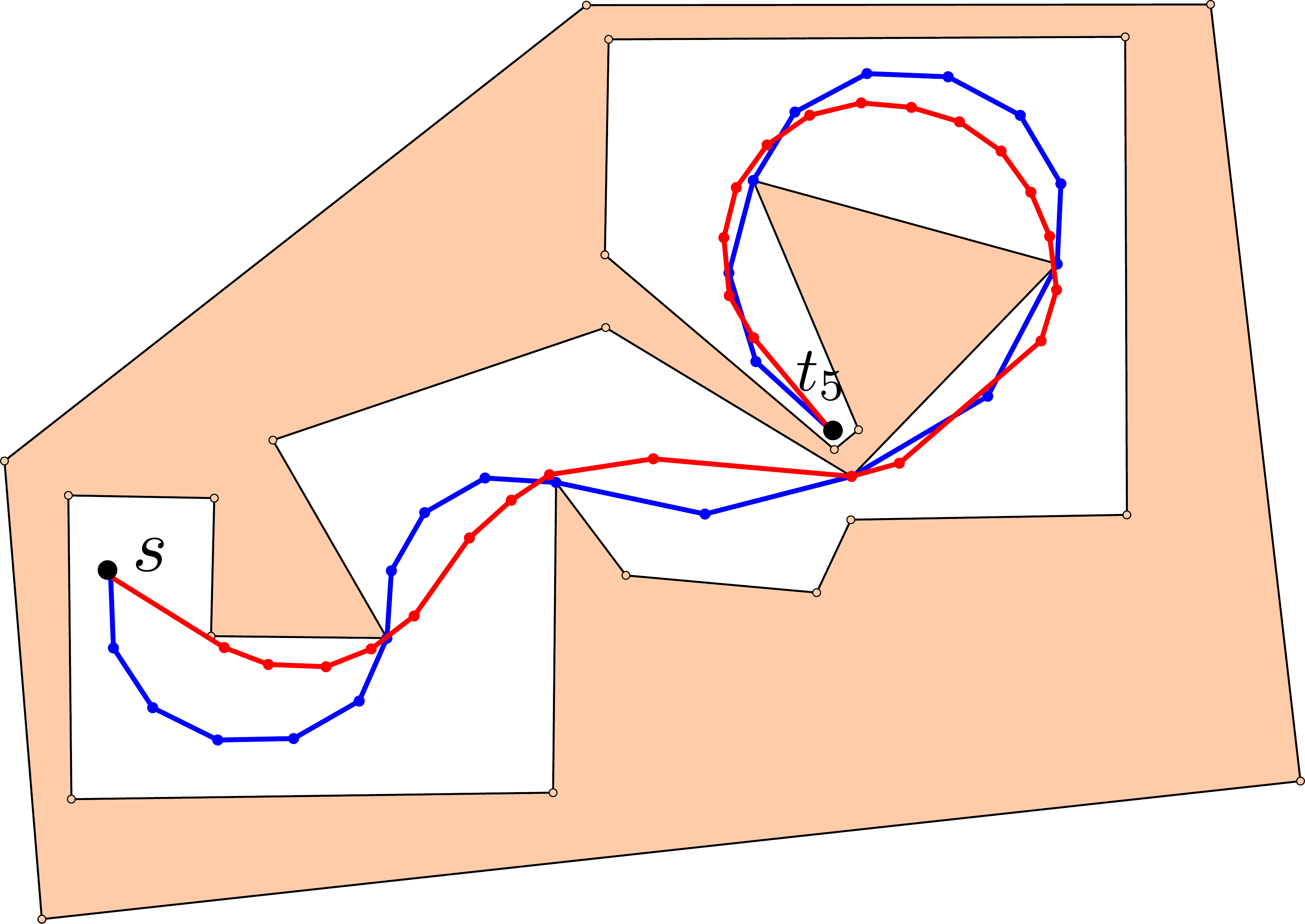}
    	\caption{}
    \end{subfigure}
	\caption{Demonstration of paths generated by the RCS (blue) and A* (red) algorithms from $s$ to $t_1,\dots, t_5$.}
	\label{fig_leaving_angle}
\end{figure*}

The numerical results are shown in Table~\ref{table_result1} and the paths generated by the RCS and A* algorithms for each target point are drawn in Figure~\ref{fig_leaving_angle}(a-e). In each figure, the blue path is produced by RCS and the red one by A*.

The first important fact derived from the results is that the running time of the A* algorithm is much larger than the RCS algorithm. As it can be seen in Table~\ref{table_result1}, in this simple test case, the running time of A* could be more than 10,000 times of the running time of RCS.

In the last column of Table~\ref{table_result1} we show the relative difference between the length of paths generated by the algorithms. The relative difference is defined as $100 \cdot |l_{\text{RCS}} - l_{\text{A*}}|/\max\{l_{\text{RCS}}, l_{\text{A*}}\}$. In this formula, $l_{\text{RCS}}$ and $l_{\text{A*}}$ are the length of the paths generated by RCS and A*, respectively. As it can be seen, the relative difference between the length of paths produced by the algorithms is small, and in this experiment is always below 5\%. Taking this into account and the fact that the running time of A* is so huge compared to RCS, the effectiveness of our method becomes apparent.

Another outcome of this experiment is that as the number of obstacle vertices around which the path must turn increases, the running time of A* increases too, whilst the running time of RCS is roughly constant. The increase in the running time of A* might also be due to the increase in the path length, but as we will see in the subsequent experiments, the impact of this increase is negligible here.

In the second experiment, we intend to see the effect of changing the maximum turning angle parameter on the 
performance of the algorithms. We used a fixed simple scene of size $400 \times 400$ consisting of a line segment as the only obstacle and two points as the starting and the target points. We chose $l = 50$ as the minimum leg length, and 4 different values for the maximum turning angle, namely $\alpha = 80, 40, 20, 10$.

The numerical results are shown in Table~\ref{table_result2} and the paths generated by RCS and A* for each value of $\alpha$ are drawn in Figure~\ref{fig_alpha}(a-d). As it can be seen from the table, the running time of the A* algorithm dramatically increases as the maximum turning angle is decreased. On the other hand, the running time of RCS is not much dependent on the value of $\alpha$.

\begin{table*}[!t]
	\caption{Illustration of the effect of changing $\alpha$ on the RCS and A* algorithms.}
	\label{table_result2}
	\centering \scriptsize
	\begin{tabular}{|C{1.2cm}|C{1.7cm}|C{1.5cm}|C{1.5cm}|C{1.5cm}|C{1cm}|C{1cm}|C{1.5cm}|}
		\hline
		{Max. turning angle (degree)} &
		{RCS\newline preprocessing time (ms)} &
		{RCS query time (ms)} &
		{RCS total time (ms)}  &
		{A* running time (ms)} &
		{RCS path length} &
		{A* path length} &
		{Relative difference} \\
		\hline
		80 & 0.20 & 0.44 & 0.64 &  119.8 & 466.1 & 477.6 & 2.4\%\\
		40 & 0.37 & 0.59 & 0.96 &  204.2 & 478.0 & 473.1 & 1.0\%\\
		20 & 0.16 & 0.68 & 0.84 & 1878.0 & 478.1 & 471.9 & 1.3\%\\
		10 & 0.18 & 0.48 & 0.66 & 5029.6 & 480.9 & 472.5 & 1.7\%\\
		\hline
	\end{tabular}
\end{table*}
	
\begin{figure*}
    \centering
	\begin{subfigure}[b]{.49\linewidth}
    \centering
		\fbox{\includegraphics[width=.5\linewidth]{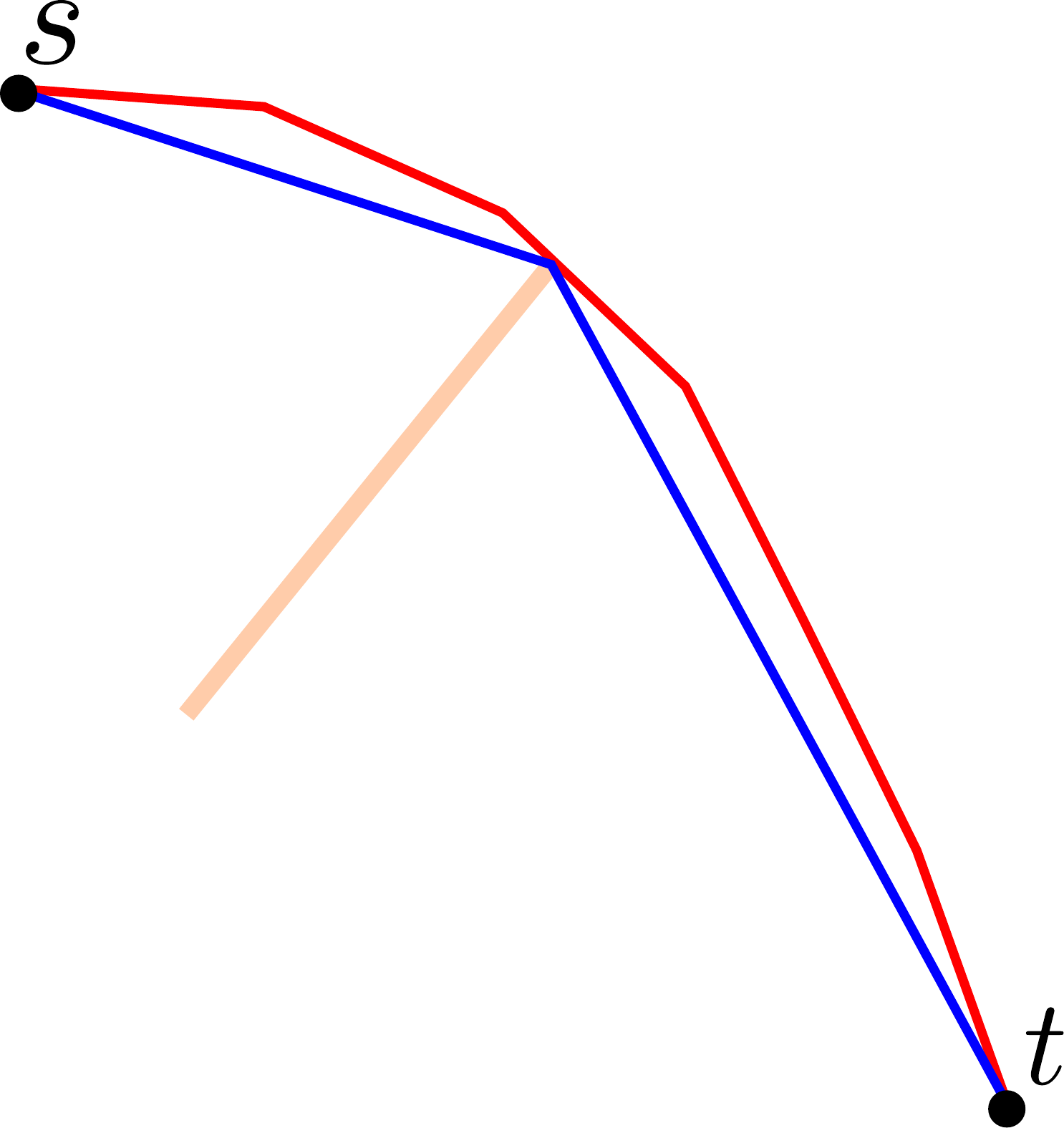}}
		\caption{$\alpha = 80^\circ$}
	\end{subfigure}
	\begin{subfigure}[b]{.49\linewidth}
		\centering
		\fbox{\includegraphics[width=.5\linewidth]{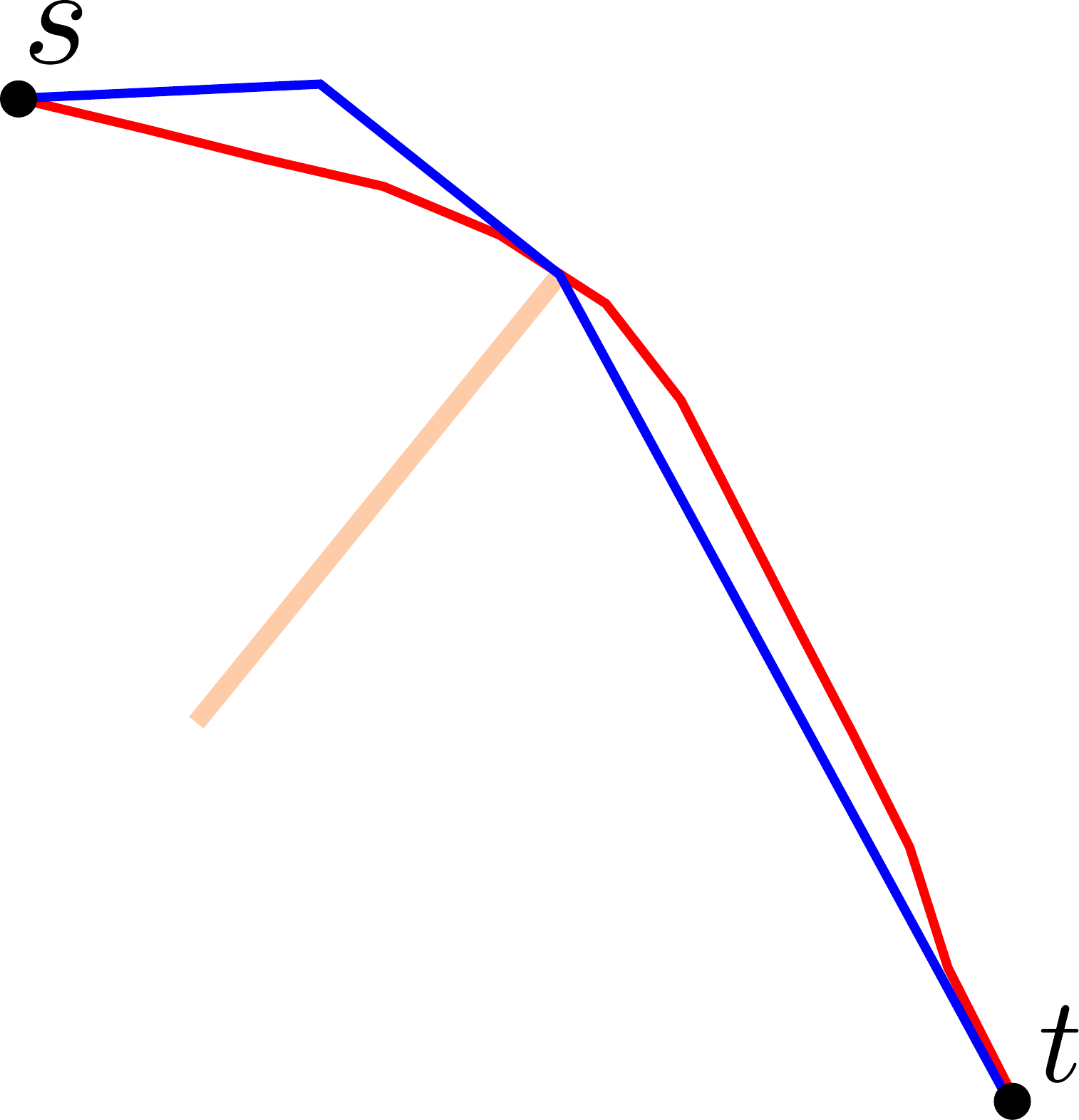}}
		\caption{$\alpha = 40^\circ$}
	\end{subfigure}

	\bigskip

	\begin{subfigure}[b]{.49\linewidth}
		\centering
		\fbox{\includegraphics[width=.5\linewidth]{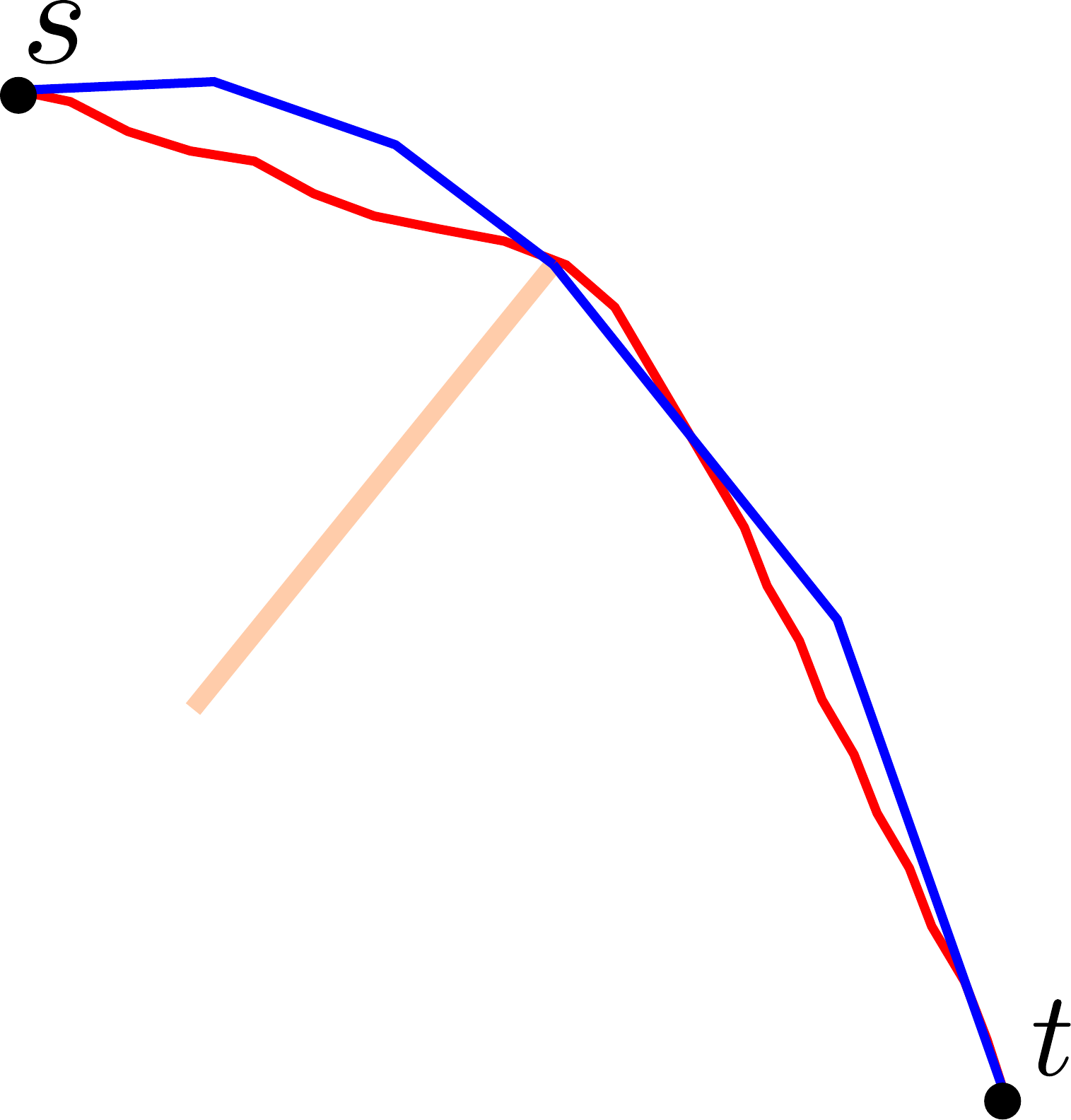}}
		\caption{$\alpha = 20^\circ$}
	\end{subfigure}
	\begin{subfigure}[b]{.49\linewidth}
		\centering
		\fbox{\includegraphics[width=.5\linewidth]{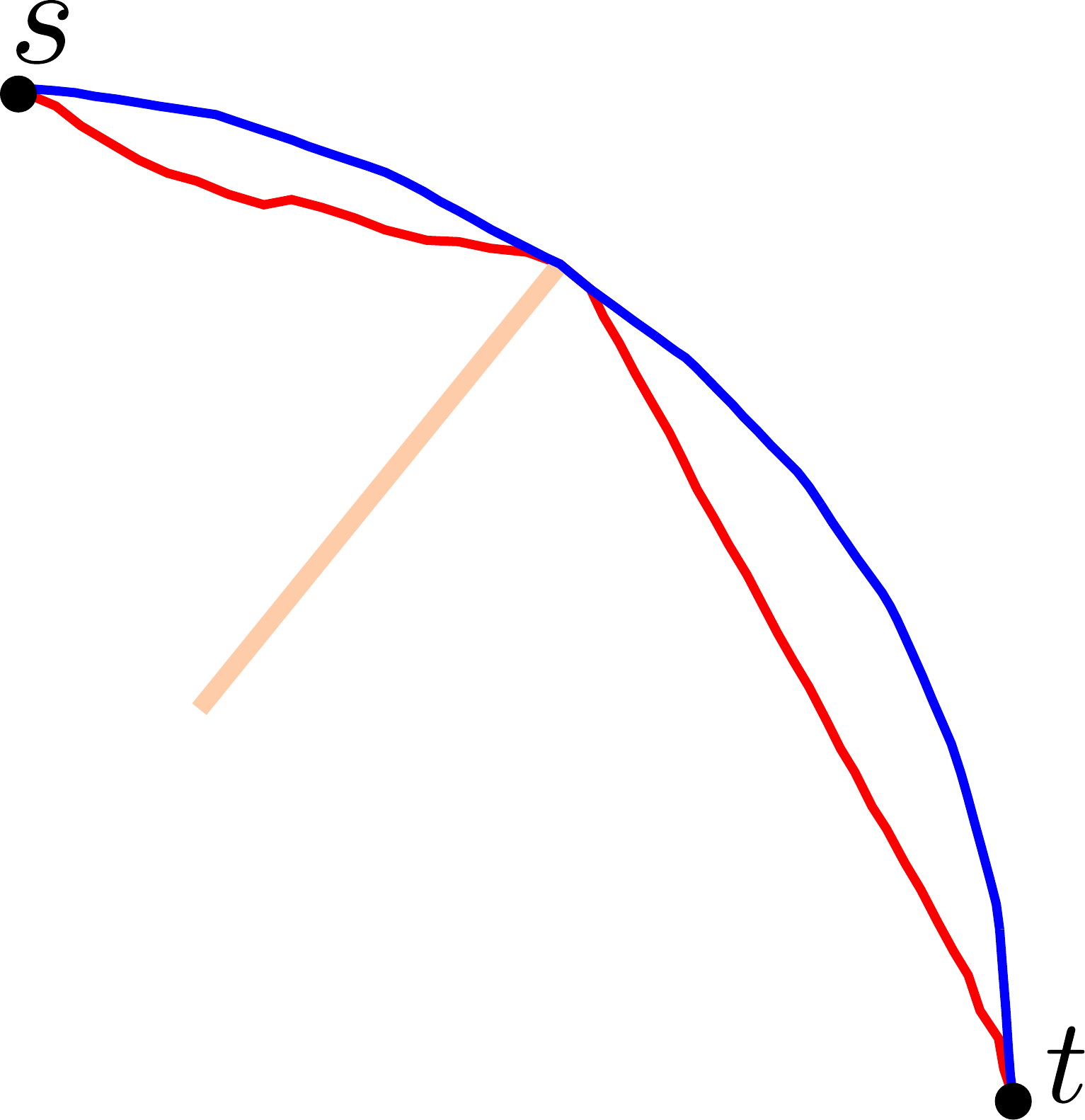}}
		\caption{$\alpha = 10^\circ$}
	\end{subfigure}
	\caption{
		Demonstration of paths generated by the RCS (blue) and A* (red) algorithms from $s$ to $t$ for different values of the maximum turning angle.}
	\label{fig_alpha}
\end{figure*}

We carry out the next experiment to see the effect of changing the minimum leg length parameter on the 
performance of the algorithms. We used a fixed simple scene of size $200 \times 200$ consisting of a line segment as the obstacle and two points as the starting and the target points. We chose $\alpha = \pi/6$ as the maximum turning angle, and 4 different values for the minimum leg length, namely $l = 40, 20, 10, 5$.

The numerical results are shown in Table~\ref{table_result3} and the paths generated by RCS and A* for each value of $l$ are drawn in Figure~\ref{fig_length}(a-d). As it can be seen from the figure, the paths generated by RCS are all the same, which is the shortest path with the given requirements. Furthermore, the running time of the A* algorithm substantially increases as the minimum leg length is decreased. However, the running time of RCS is not much dependent on the value of $\alpha$.

\begin{table*}[!t]
	\caption{Illustration of the effect of changing $l$ on the RCS and A* algorithms.}
	\label{table_result3}
	\centering \scriptsize
	\begin{tabular}{|C{1.2cm}|C{1.7cm}|C{1.5cm}|C{1.5cm}|C{1.5cm}|C{1cm}|C{1cm}|C{1.5cm}|}
		\hline
		{Min. leg length} &
		{RCS\newline preprocessing time (ms)} &
		{RCS query time (ms)} &
		{RCS total time (ms)}  &
		{A* running time (ms)} &
		{RCS path length} &
		{A* path length} &
		{Relative difference} \\
		\hline
		40 & 0.16 & 0.72 & 0.88 & 121.7 & 279.2 & 283.3 & 1.4\%\\
		20 & 0.20 & 0.75 & 0.95 & 217.8 & 279.2 & 282.0 & 1.0\%\\
		10 & 0.22 & 0.73 & 0.95 & 554.4 & 279.2 & 281.9 & 1.0\%\\
		 5 & 0.22 & 0.82 & 1.04 & 789.6 & 279.2 & 283.1 & 1.4\%\\
		\hline
	\end{tabular}
\end{table*}

\begin{figure*}
	\centering
	\begin{subfigure}[b]{.49\linewidth}
		\centering
		\fbox{\includegraphics[width=.5\linewidth]{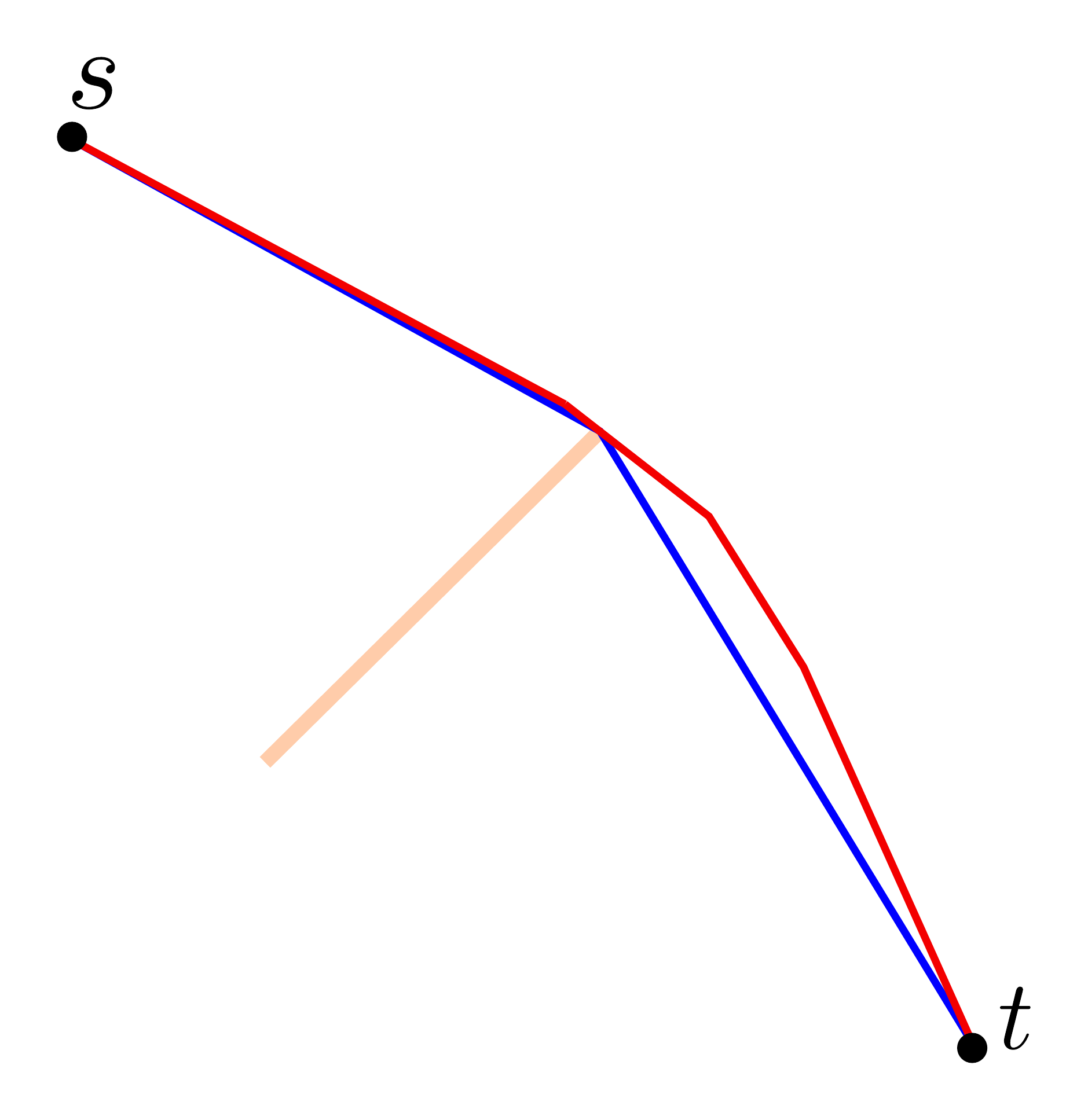}}
		\caption{$l=40$}
	\end{subfigure}
	\begin{subfigure}[b]{.49\linewidth}
		\centering
		\fbox{\includegraphics[width=.5\linewidth]{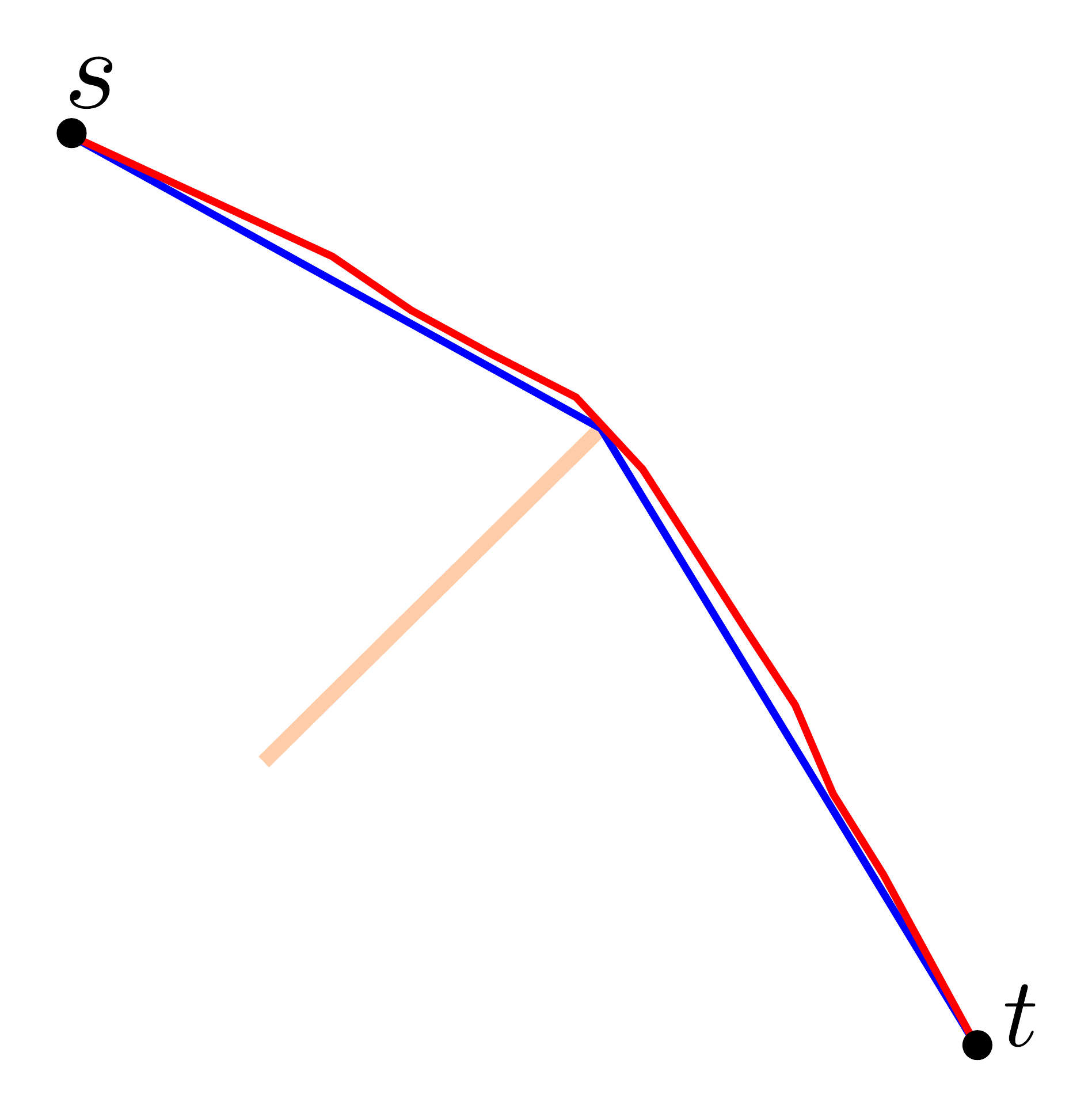}}
		\caption{$l=20$}
	\end{subfigure}
	
	\bigskip
	
	\begin{subfigure}[b]{.49\linewidth}
		\centering
		\fbox{\includegraphics[width=.5\linewidth]{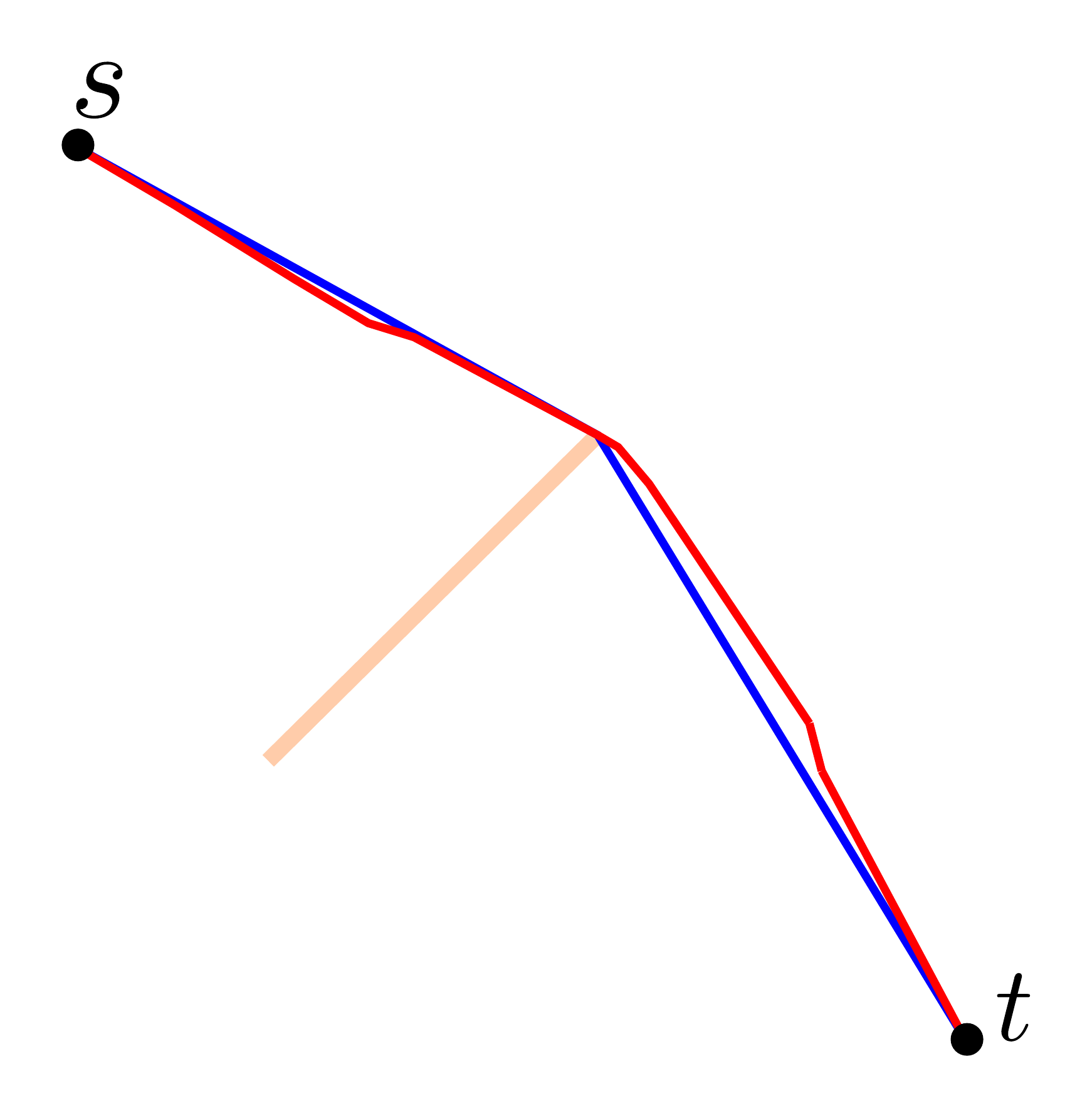}}
		\caption{$l=10$}
	\end{subfigure}
	\begin{subfigure}[b]{.49\linewidth}
		\centering
		\fbox{\includegraphics[width=.5\linewidth]{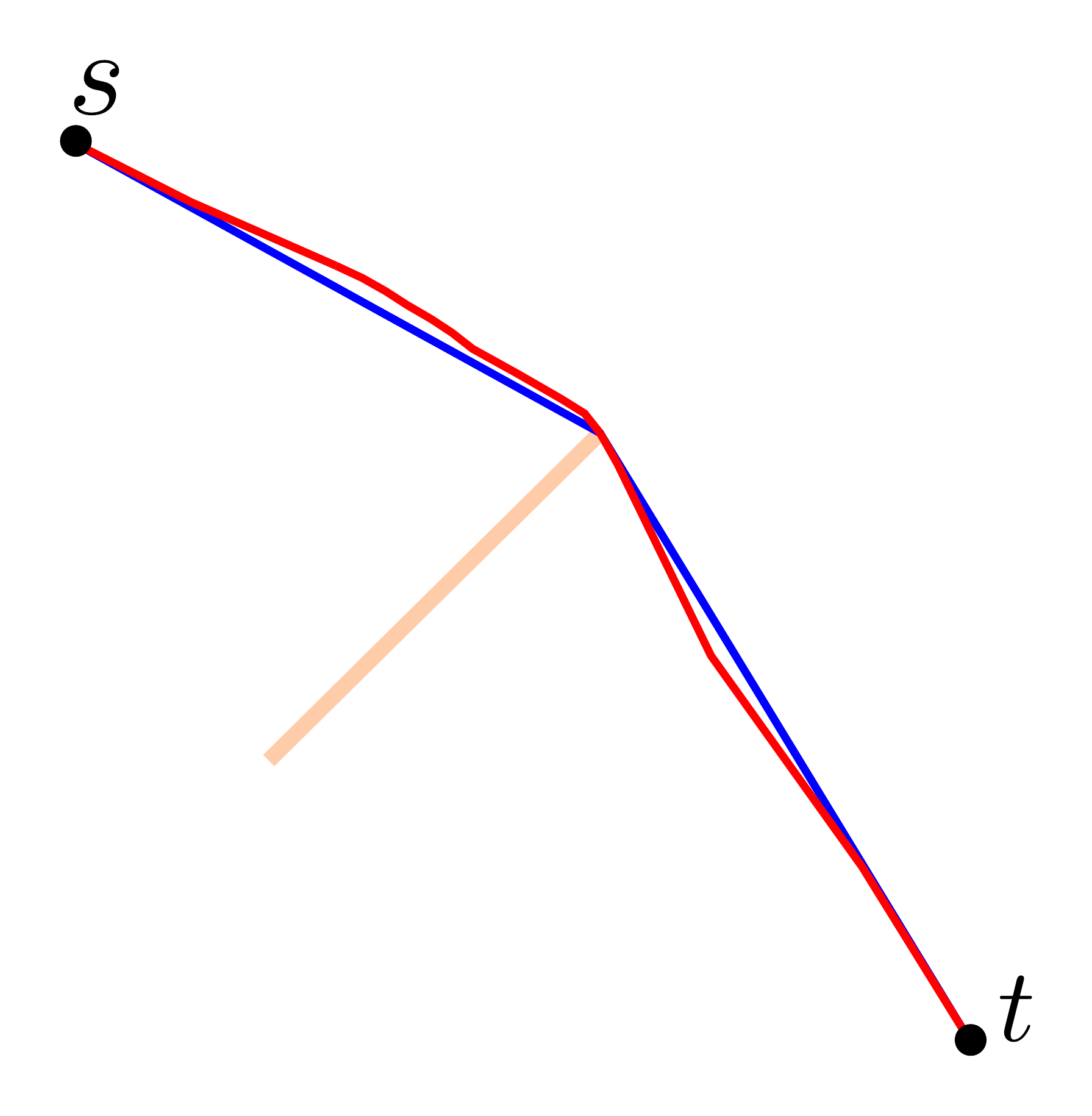}}
		\caption{$l=5$}
	\end{subfigure}
	\caption{
		Demonstration of paths generated by the RCS (blue) and A* (red) algorithms from $s$ to $t$ for different values of the minimum leg length.}
	\label{fig_length}
\end{figure*}

In the fourth experiment, we want to illustrate the effect of increasing the complexity of the scene in terms of the number of vertices of obstacles. In this experiment, each scene consists of a fixed staring and a target point, and a set of rod-shaped obstacles. We double the number of obstacles in each scene. The maximum turning angle is selected as $\alpha = \pi/6$, and the minimum leg length as $l = 50$.

The results are shown in Table~\ref{table_result4} and Figure~\ref{fig_Obstacles}(a-e). From the results of the table, we see that the dependency of the running time of RCS is much greater to the complexity of the scene than A*. Despite that, the running time of RCS is still less than the running time of A*. The relative difference of the two algorithms are in most cases very small, and in three out of five cases, RCS found a path shorter than the path found by A*.

\begin{table*}[!t]
	\caption{Illustration of the effect of increasing the complexity of scene.}
	\label{table_result4}
	\centering \scriptsize
	\begin{tabular}{|C{0.6cm}|C{1.4cm}|C{1.7cm}|C{1.5cm}|C{1.5cm}|C{1.5cm}|C{1cm}|C{1cm}|C{1.3cm}|}
		\hline
		Test Case &
		\#Obstacle &
		{RCS\newline preprocessing time (ms)} &
		{RCS query time (ms)} &
		{RCS total time (ms)}  &
		{A* running time (ms)} &
		{RCS path length} &
		{A* path length} &
		{Relative difference} \\
		\hline
		$t_1$ & 2  & 1.18   & 0.94 & 2.12   & 188.3  & 643.1 & 645.7 & 0.4\%\\
		$t_2$ & 4  & 6.56   & 1.18 & 7.74   & 149.6  & 643.2 & 645.8 & 0.4\%\\
		$t_3$ & 8  & 14.45  & 1.31 & 15.76  & 449.9  & 664.9 & 663.4 & 0.2\%\\
		$t_4$ & 16 & 33.08  & 2.07 & 35.15  & 318.8  & 747.3 & 669.1 & 10.5\%\\
		$t_5$ & 32 & 209.36 & 4.00 & 213.36 & 762.0  & 740.5 & 756.1 & 2.1\%\\
		\hline
	\end{tabular}
\end{table*}

\begin{figure*}
	\centering
	\begin{subfigure}[b]{.49\linewidth}
		\centering
		\fbox{\includegraphics[width=.5\linewidth]{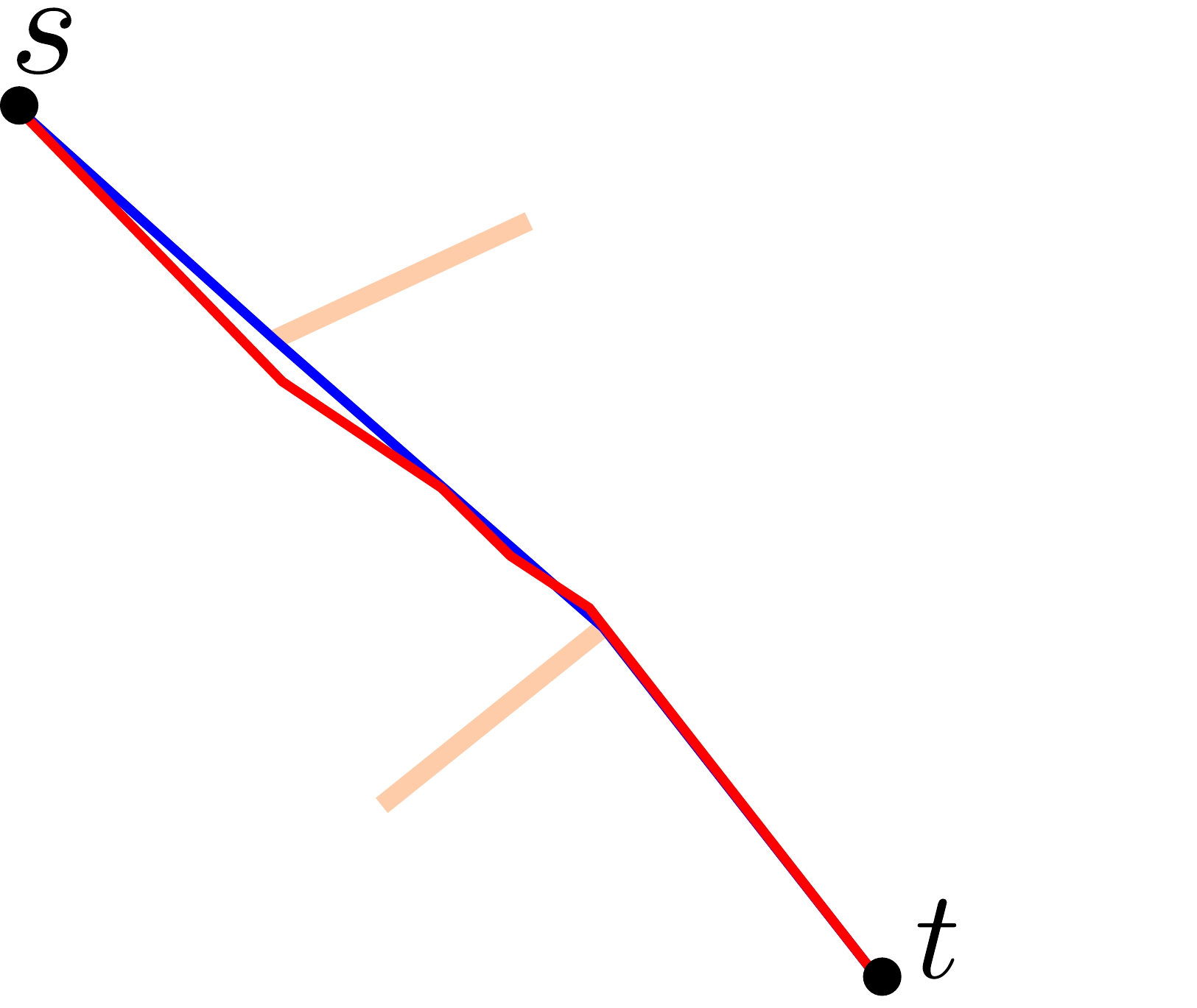}}
		\caption{$t_1$}
	\end{subfigure}
	\begin{subfigure}[b]{.49\linewidth}
		\centering
		\fbox{\includegraphics[width=.5\linewidth]{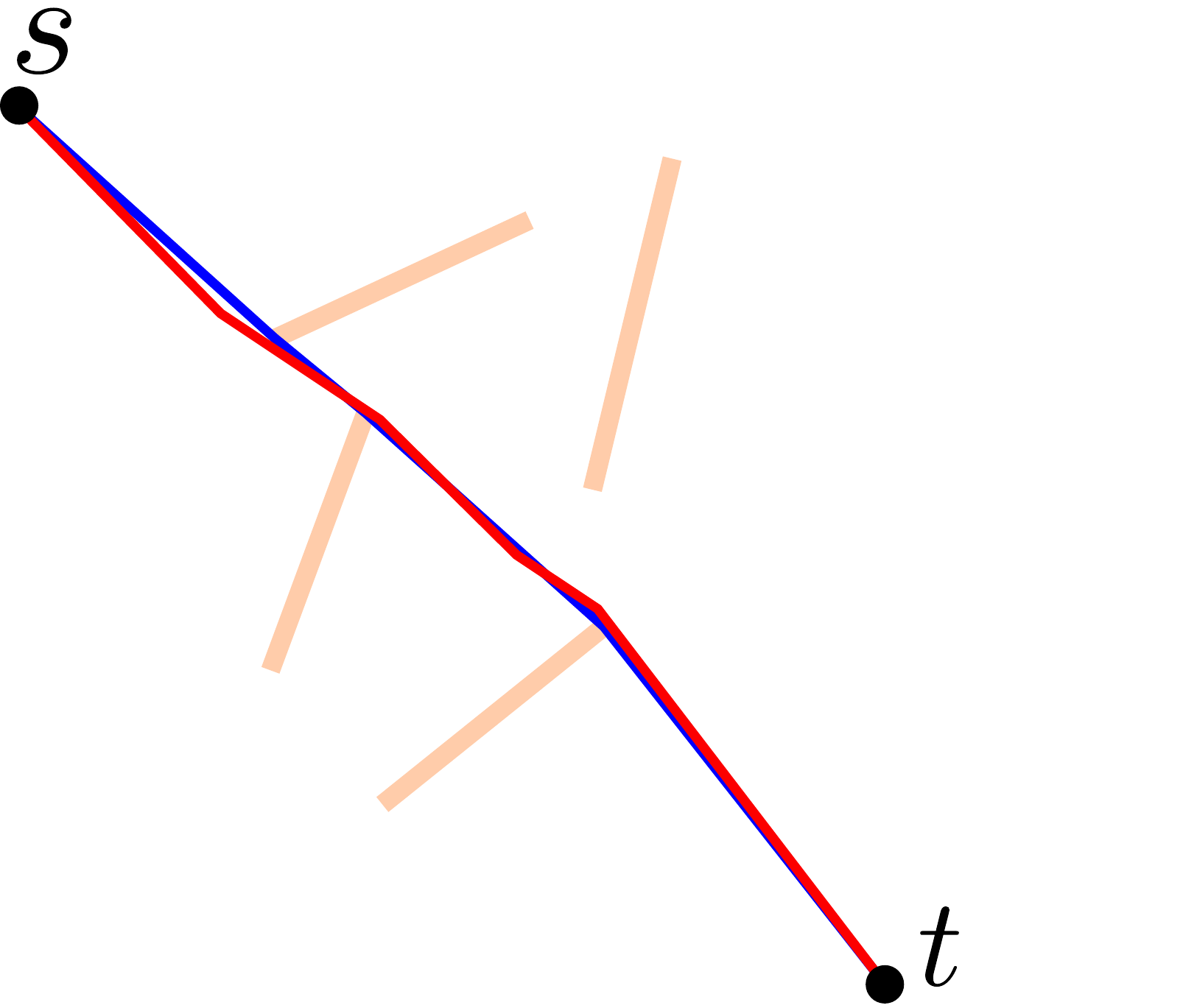}}
		\caption{$t_2$}
	\end{subfigure}
	
	\bigskip
	
	\begin{subfigure}[b]{.49\linewidth}
		\centering
		\fbox{\includegraphics[width=.5\linewidth]{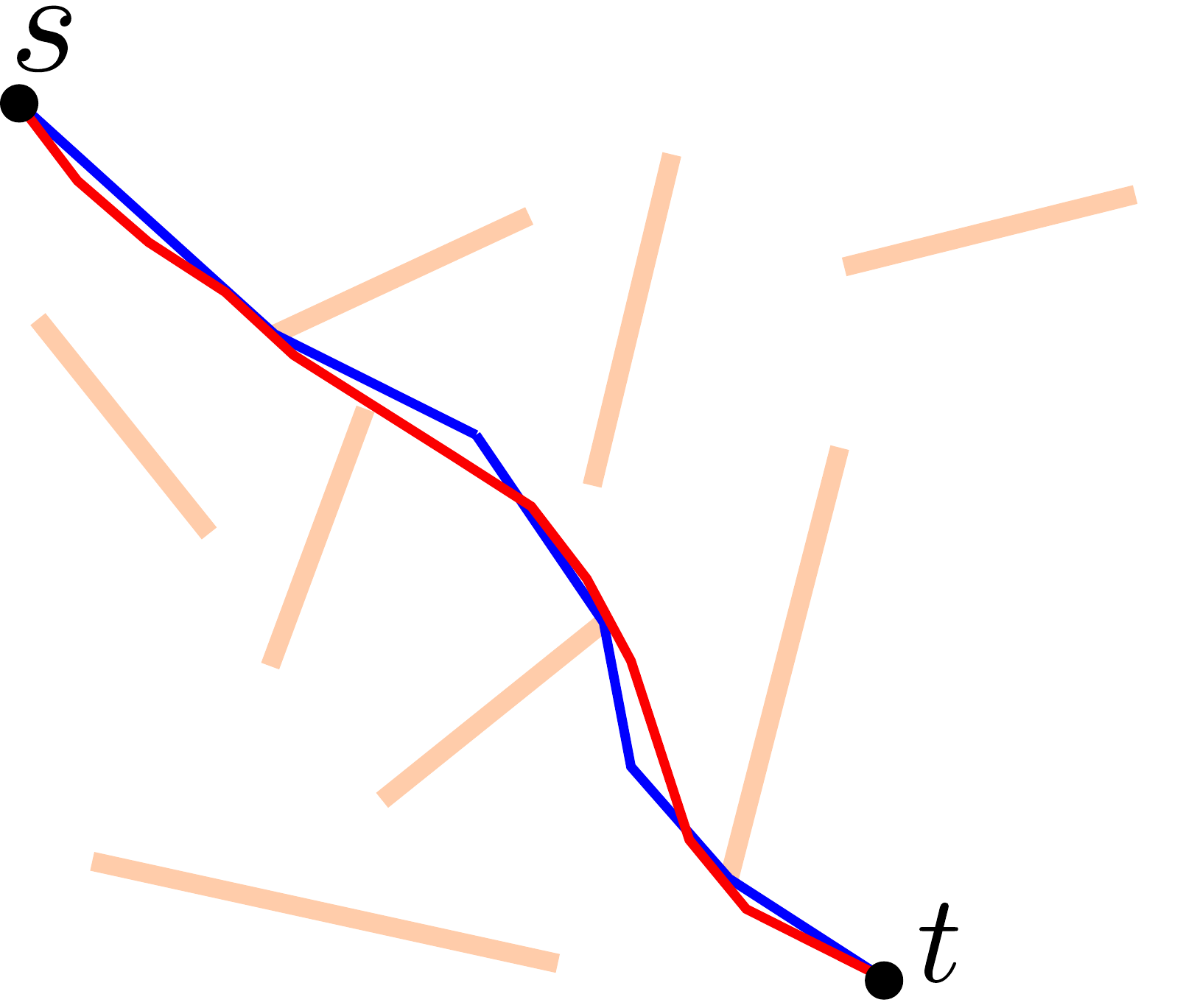}}
		\caption{$t_3$}
	\end{subfigure}
	\begin{subfigure}[b]{.49\linewidth}
		\centering
		\fbox{\includegraphics[width=.5\linewidth]{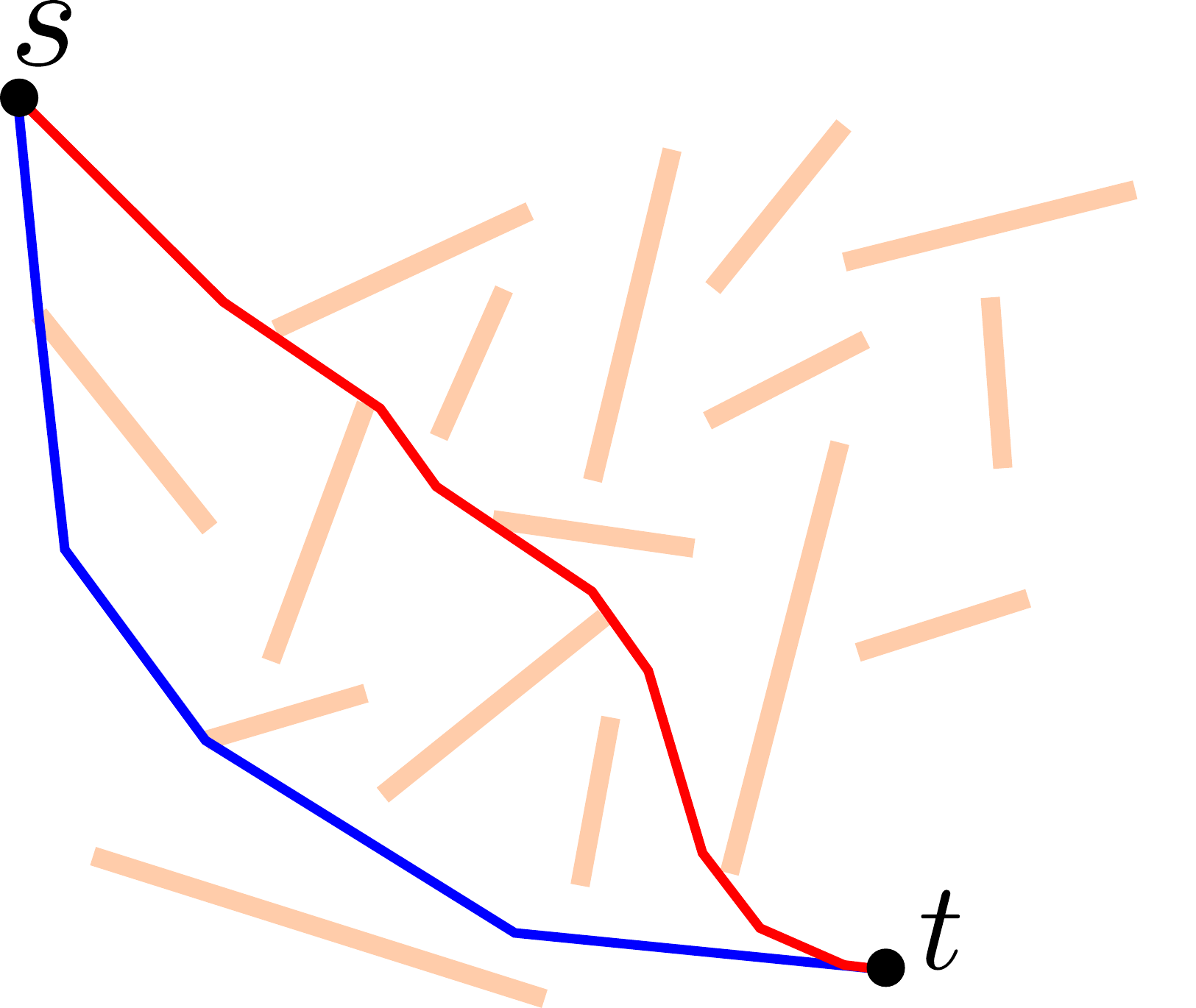}}
		\caption{$t_4$}
	\end{subfigure}

\bigskip

	\begin{subfigure}[b]{.49\linewidth}
		\centering
		\fbox{\includegraphics[width=.5\linewidth]{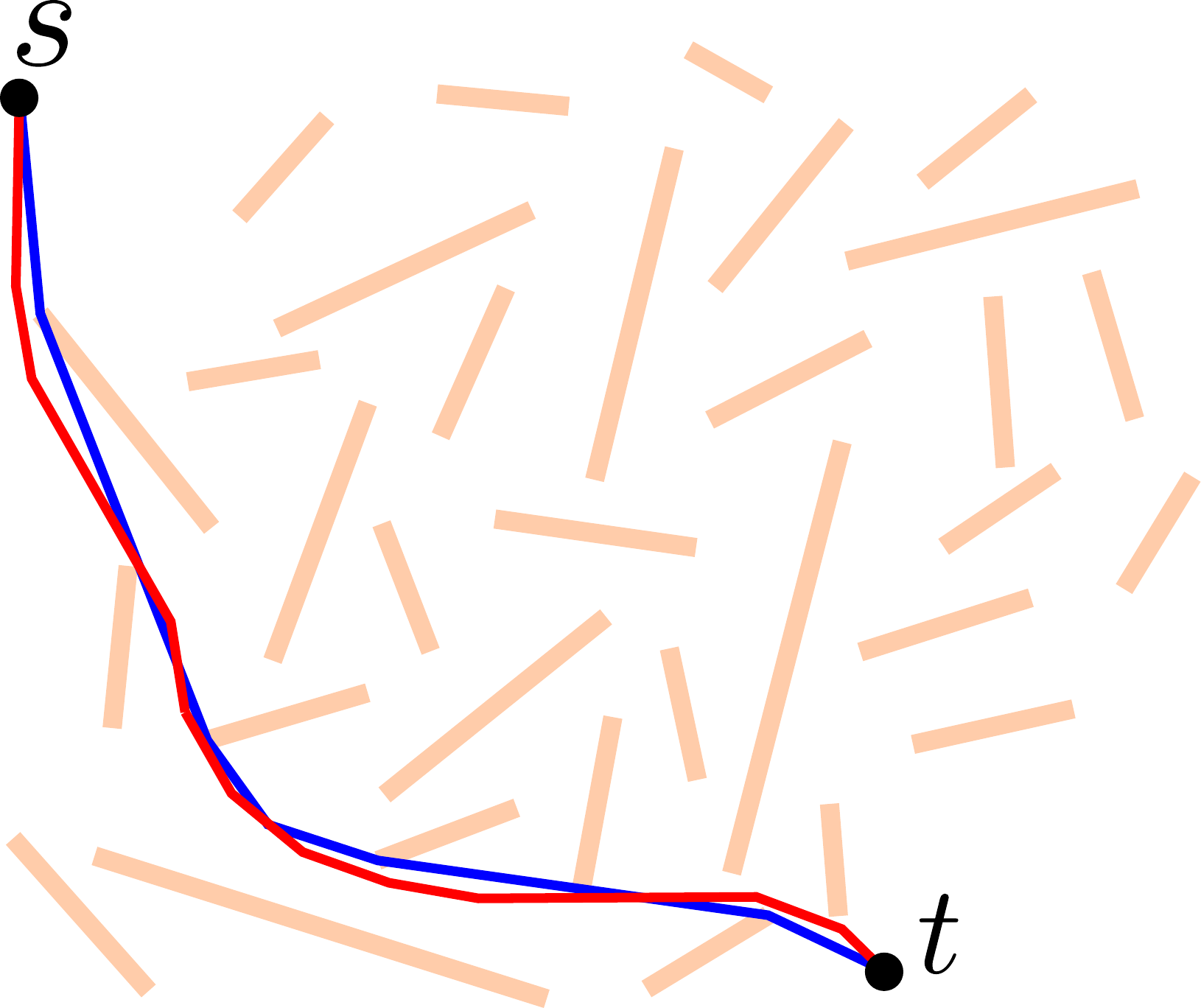}}
		\caption{$t_5$}
	\end{subfigure}
	\caption{
		Demonstration of paths generated by the RCS (blue) and A* (red) algorithms from $s$ to $t$ for scenes of different complexities.}
	\label{fig_Obstacles}
\end{figure*}

In the fifth experiment, we show the effect of rescaling the workspace on the running time of the two algorithms. As illustrated in Figure~\ref{fig_Ratio} the scene consists of a fixed starting and a target point and a single rod-shaped obstacle. In each test case, we rescale the scene by doubling each coordinate, as well as the minimum leg length $l$. The maximum turning angle is selected as $\alpha = \pi/6$. It is obvious that with rescaling the scene, and the minimum leg length parameter, the path will not change.

The results are shown in Table~\ref{table_result5}. From these results we can easily conclude that the running time of A* is greatly dependent on the size of the scene, while the running time of RCS is independent.

\begin{table*}[!t]
	\caption{Illustration of the effect of changing size of the scene.}
	\label{table_result5}
	\centering \scriptsize
	\begin{tabular}{|C{0.6cm}|C{1.2cm}|C{1.3cm}|C{1.3cm}|C{1.3cm}|C{1.3cm}|C{1.5cm}|C{1cm}|C{1cm}|C{1.3cm}|}
		\hline
		{Test Case} &
		{Scene size} &
		{Minimum leg length} &
		{RCS\newline prep. time (ms)} &
		{RCS query time (ms)} &
		{RCS total time (ms)}  &
		{A* running time (ms)} &
		{RCS path length} &
		{A* path length} &
		{Relative difference} \\
		\hline
		$t_1$ & $100 \times 100$ & 10 & 0.22 & 0.77 & 0.99  & 182.2  & 141.6  & 143.4  & 1.26\%\\
		$t_2$ & $200 \times 200$ & 20 & 0.27 & 0.79 & 1.06  & 233.2  & 284.7  & 287.6  & 1.01\%\\
		$t_3$ & $400 \times 400$ & 40 & 0.23 & 0.80 & 1.03  & 455.0  & 570.8  & 574.5  & 0.64\%\\
		$t_4$ & $800 \times 800$ & 80 & 0.22 & 0.80 & 1.02  & 1178.9 & 1143.0 & 1148.8 & 0.50\%\\
		\hline
	\end{tabular}
\end{table*}

\begin{figure*}
	\centering
	\begin{subfigure}[b]{.49\linewidth}
		\centering
		\fbox{\includegraphics[width=.2\linewidth]{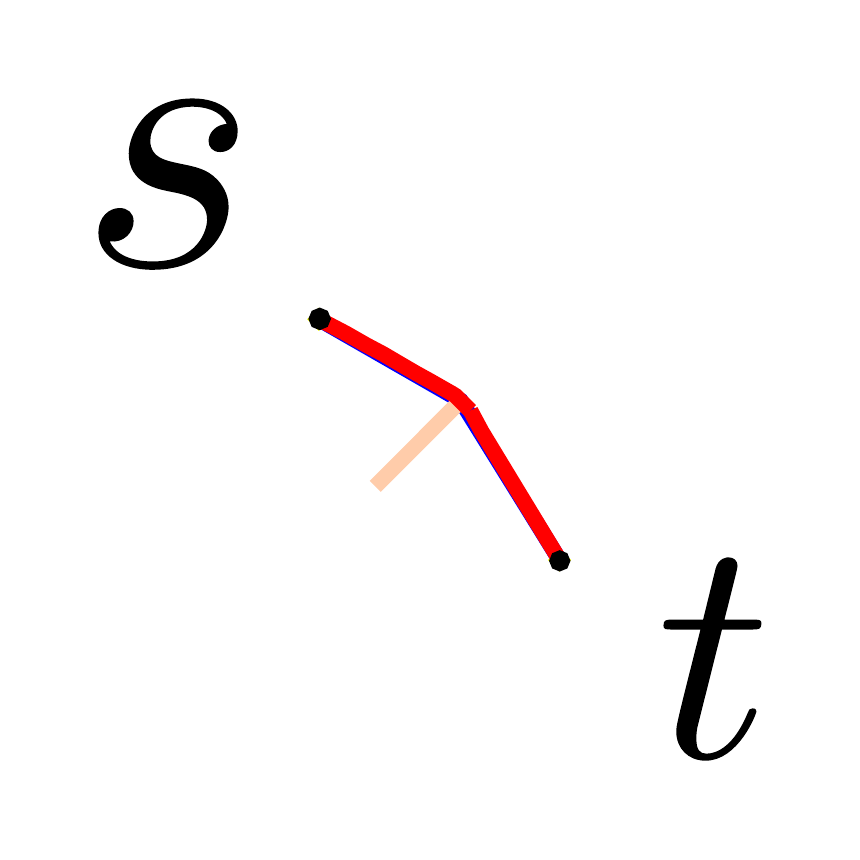}}
		\caption{Scene of size $100 \times 100$.}
	\end{subfigure}
	\begin{subfigure}[b]{.49\linewidth}
		\centering
		\fbox{\includegraphics[width=.3\linewidth]{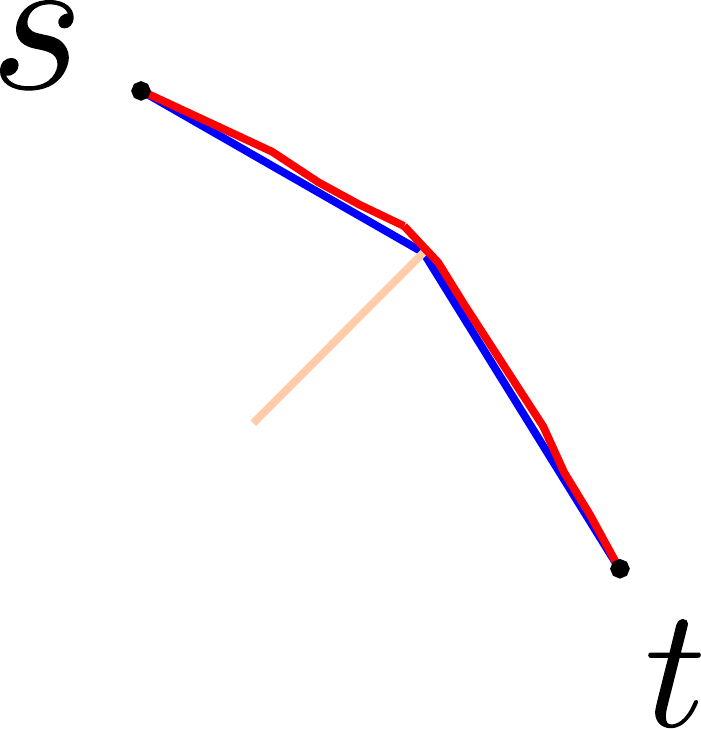}}
		\caption{Scene of size $200 \times 200$.}
	\end{subfigure}
	
	\bigskip
	
	\begin{subfigure}[b]{.49\linewidth}
		\centering
		\fbox{\includegraphics[width=.5\linewidth]{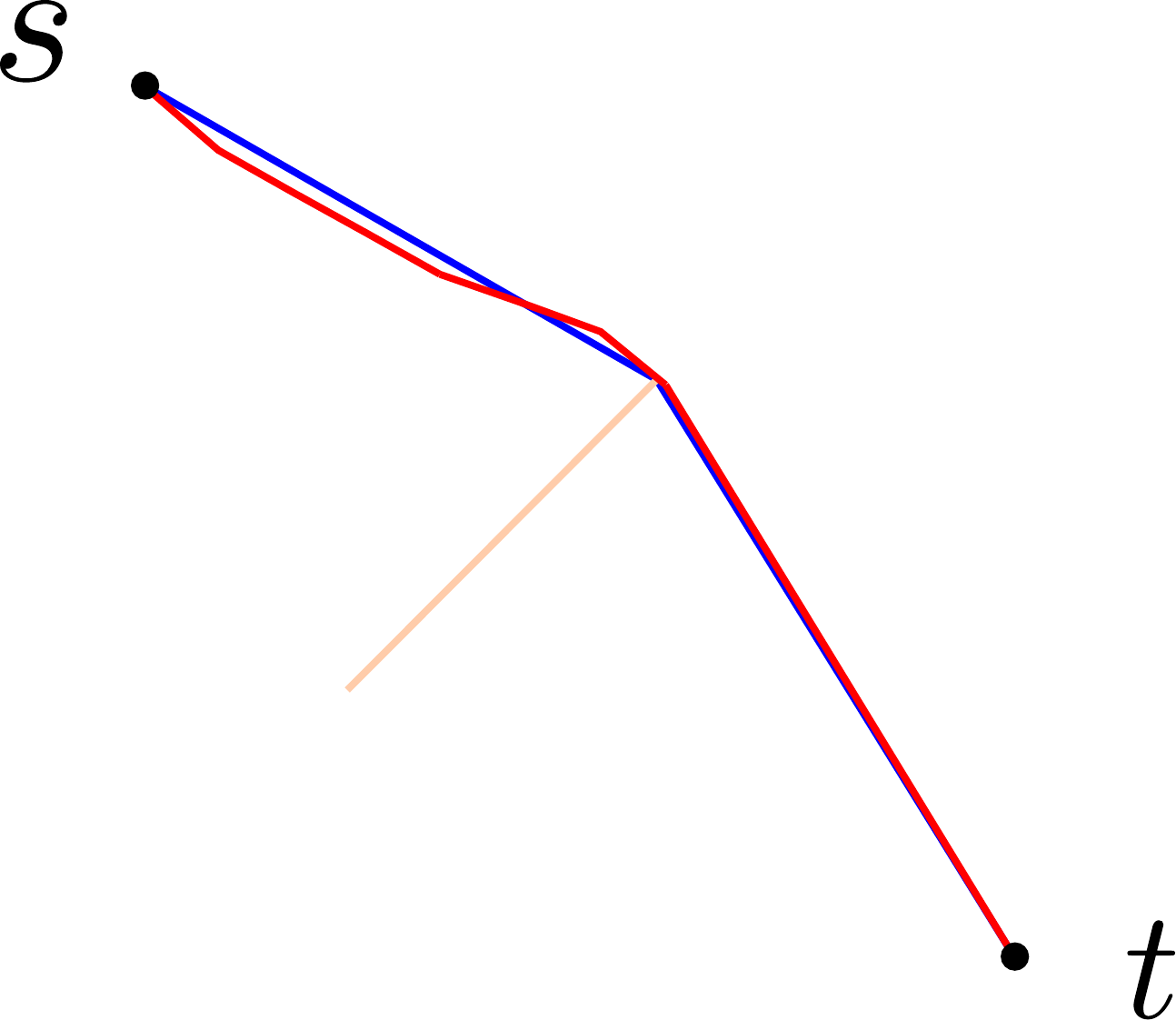}}
		\caption{Scene of size $400 \times 400$.}
	\end{subfigure}
	\begin{subfigure}[b]{.49\linewidth}
		\centering
		\fbox{\includegraphics[width=.9\linewidth]{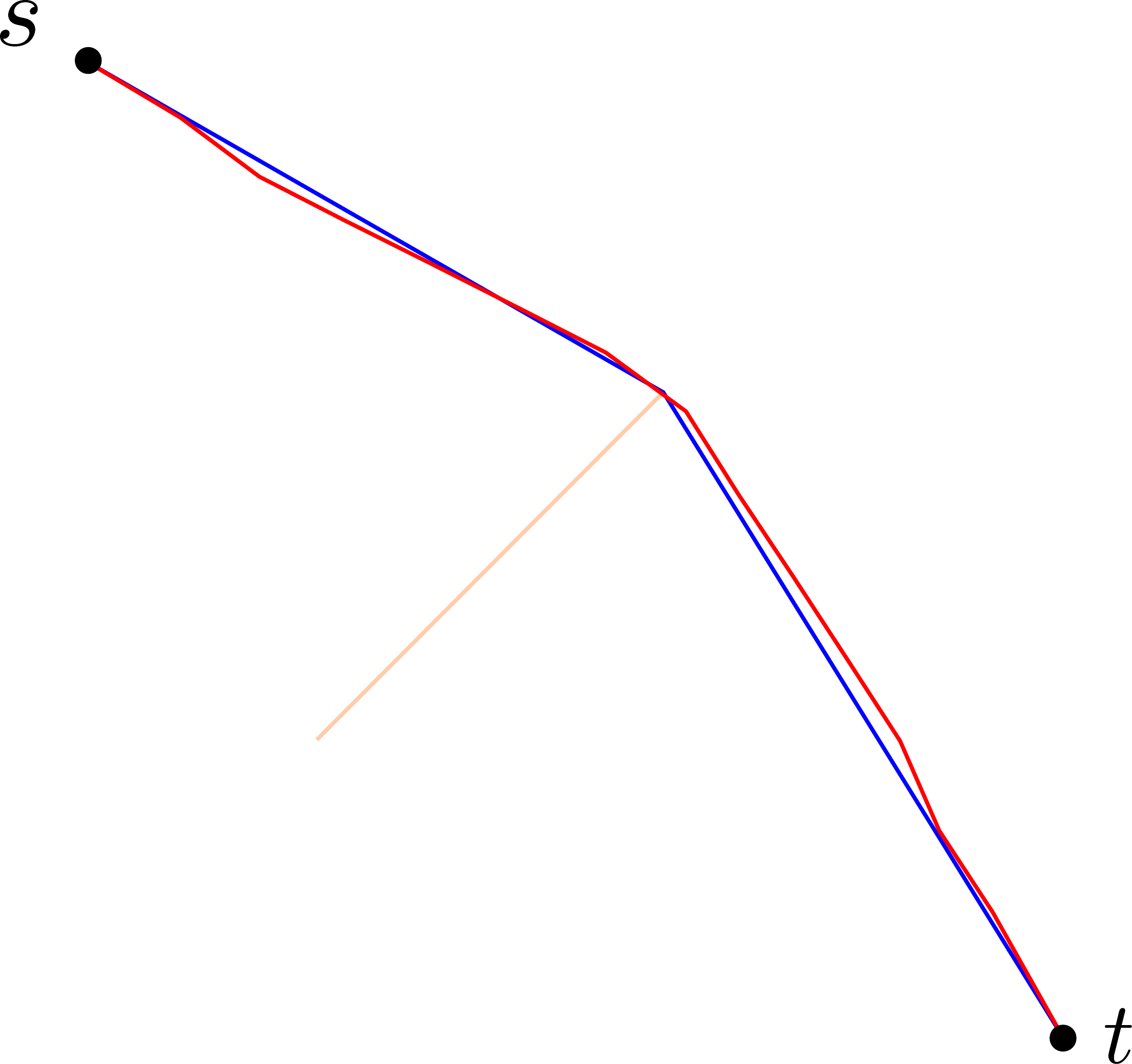}}
		\caption{Scene of size $800 \times 800$.}
	\end{subfigure}
	\caption{Demonstration of paths generated by the RCS (blue) and A* (red) algorithms from $s$ to $t$ for scenes of different sizes.}
	\label{fig_Ratio}
\end{figure*}

In the sixth experiment, we demonstrate the effect of path length on the running time of the algorithms. In this experiment, we use a simple scene consisting of only a starting point $s$ and a target point $t$. The difference between each test case is the distance between the starting and the target points. In this experiment we choose $l = 50$, and $\alpha = \pi/6$. It is obvious that in this experiment the optimum path is simply the line segment from $s$ to $t$.

The results of this experiment are shown in Table~\ref{table_result6}. From these results, it can easily be confirmed that the running time of the two algorithms is not dependent on the path length.

\begin{table*}[!t]
	\caption{Illustration of the effect of the path length on the running time of RCS and A*.}
	\label{table_result6}
	\centering \scriptsize
	\begin{tabular}{|C{1.5cm}|C{2.5cm}|C{1.5cm}|C{1.7cm}|C{1.7cm}|}
		\hline
		{Path length} &
		{RCS preprocessing time (ms)} &
		{RCS query time (ms)} &
		{RCS total time (ms)}  &
		{A* running time (ms)} \\
		\hline
		100 & 0.0046 & 0.0159 & 0.0205 &  94.19  \\
		200 & 0.0046 & 0.0162 & 0.0208 & 104.21  \\
		400 & 0.0046 & 0.0174 & 0.0220 & 100.21  \\
		800 & 0.0046 & 0.0179 & 0.0225 &  95.73  \\
		\hline
	\end{tabular}
\end{table*}

For the last experiment, we want to test the scalability of the RCS algorithm. We use six different polygonal domains of various sizes, and run the algorithm to find the required preprocessing and query time of the algorithm for each one. In these test cases, we choose the maximum turning angle $\alpha = \pi / 9$ and the minimum leg length $l = 60$. For each test case, the number of obstacles and vertices, and the preprocessing and query time in milliseconds are shown in Table~\ref{table_result7}.

One of the configurations used in this experiment is depicted in Figure~\ref{testcase3}. In this figure, a polygonal domain with 9 obstacles and 39 vertices is illustrated. The result of the algorithm on this configuration is shown when $\alpha=\pi/18$ and $l=60$.

\begin{table*}[!t]
	\renewcommand{\arraystretch}{1.3}
	\caption{Preprocessing and query time of the RCS algorithm for scenes of various complexities.}
	\label{table_result7}
	\centering
	\begin{tabular}{|c|c|c|c|c|}
		\hline
		Test case & \#Obstacles & \#Vertices & Preprocessing time (ms) & Query time (ms) \\
		\hline
		1 & 1 & 3 & 2 & 3 \\ 
		2 & 2 & 7 & 25 & 3 \\ 
		3 & 9 & 39 & 1166 & 18 \\ 
		4 & 26 & 100 & 2033 & 26 \\ 
		5 & 84 & 300 & 28860 & 188 \\ 
		6 & 180 & 500 & 223134 & 570 \\ 
		\hline
	\end{tabular}
\end{table*}

\begin{figure}[!t]
	\centering
	\fbox{\includegraphics[width=2.5in]{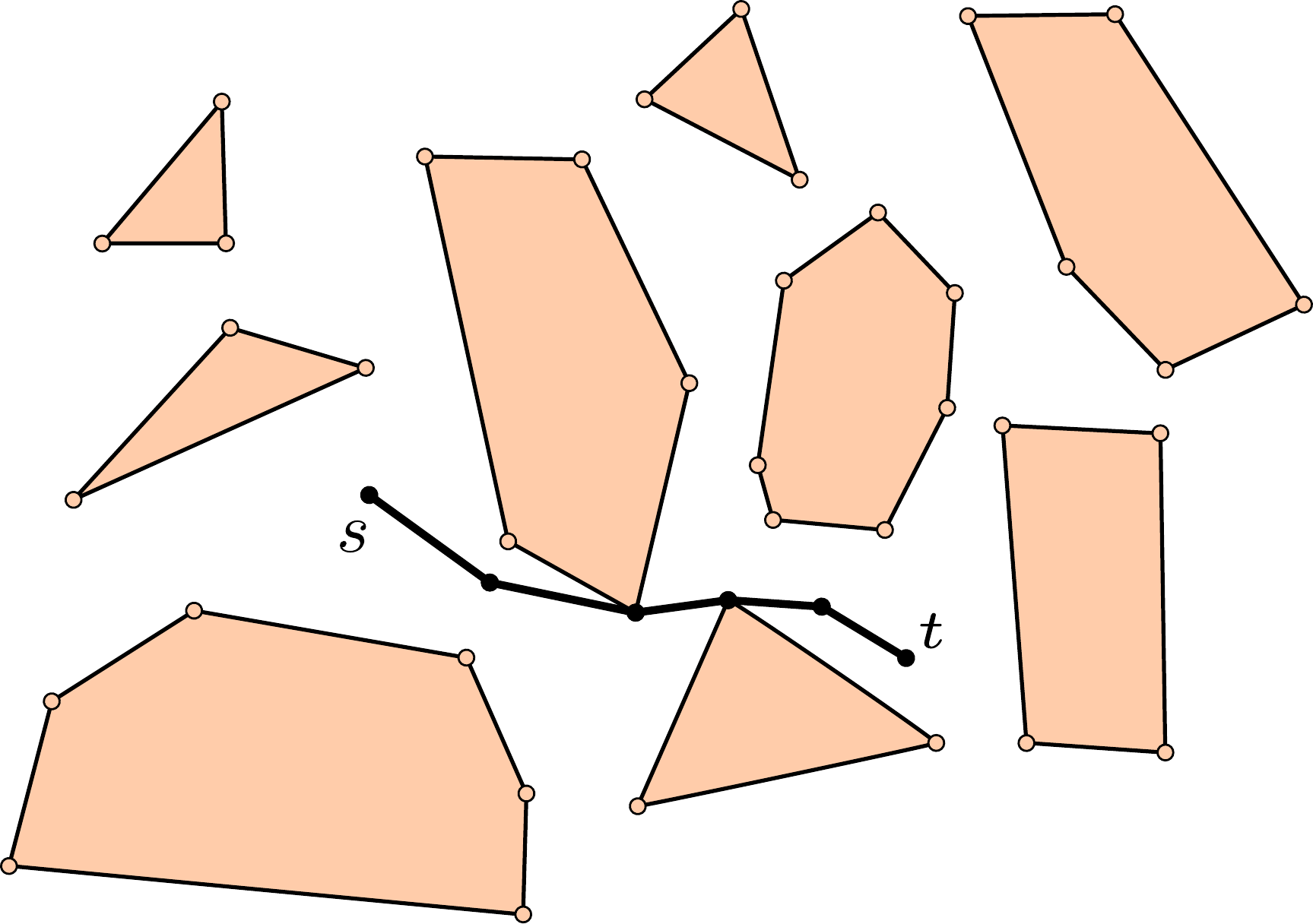}}
	\caption{A polygonal domain with 9 obstacles and 39 vertices is illustrated in which we find a path with $\alpha=\pi/18$ and $l=60$.}
	\label{testcase3}
\end{figure}

\section{Conclusion} \label{Conclusion}
In this paper we studied the problem of path planning for a robot with maximum turning angle and minimum leg length in a polygonal domain. We first proposed an algorithm to find a path avoiding obstacles with the aforementioned requirements. We then proved that the required space and the running time of the algorithm are $O(n^2)$ and $O(n^2 \log n)$, respectively. In order to decrease the path planning time, we further decomposed the algorithm into two parts, the preprocessing phase and the query phase.  This helps to reduce the running time of the algorithm when a set of target points are going to be used in a fixed polygonal domain. In another extension, we changed the algorithm to answer path planning queries that require the object to approach the target from a given direction (or a range of directions). 

Although our main goal was not to find the shortest path, as presented by the experiments, the solution is not much longer than the shortest path. Furthermore, comparing to most of the previous work on this problem, which uses heuristic and evolutionary methods, the running time of the algorithm is much less, and the algorithm is greatly scalable to large complex environments. For future work, we try to improve the algorithm and find the shortest path with the given requirements.

\bibliographystyle{plain}
\bibliography{planning}
\end{document}